\documentclass{article}
\usepackage{amsmath}
\usepackage{amsfonts}
\usepackage{amssymb}
\usepackage{graphicx}
\usepackage{fullpage}
\usepackage{comment}
\usepackage{algorithmic}
\usepackage{xcolor}
\usepackage{algorithm}
\usepackage{amsthm}
\usepackage[utf8]{inputenc}
\usepackage[english]{babel}
\usepackage{enumitem}
\usepackage{thm-restate}
\newcommand{\x}{\textbf{x}}

\newcommand{\real}{\mathbb{R}}
\newcommand{\w}{\textbf{w}}
\renewcommand{\S}{\mathcal{S}}
\newcommand{\reals}{\mathbb{R}}

\newtheorem{lemma}{Lemma}
\newtheorem{definition}{Definition}

\newtheorem{corollary}{Corollary}

\usepackage{lmodern}
\usepackage[utf8]{inputenc} 
\usepackage[T1]{fontenc}    
\usepackage{url}            
\usepackage{booktabs}       
\usepackage{amsfonts}       
\usepackage{nicefrac}       
\usepackage{microtype}      
\usepackage{dsfont}
\usepackage{bbm}
\usepackage{xcolor}
\usepackage{tikz}
\usetikzlibrary{calc,angles,positioning,intersections,quotes,decorations.markings}
\usepackage{pgfplots}

\title{Algorithmic guarantees for Inverse Imaging \\ with Untrained Network Priors}

\author{Gauri~Jagatap and Chinmay~Hegde\thanks{The authors are with the Department of Electrical and Computer Engineering, New York University, NY, USA. Email: {gauri.jagatap,chinmay.h}@nyu.edu. This work was supported in part by NSF grants CCF-1566281, CAREER CCF-1750920, CCF-1815101, GPU grants from NVIDIA Corporation, and a faculty fellowship from the Black and Veatch Foundation. This work appears in the NeurIPS conference proceedings as: G. Jagatap and C. Hegde, ``Algorithmic guarantees for inverse imaging with untrained network priors.", Thirty-third Conference on Neural Information Processing Systems, 2019.	
} 	
}
\date{}
\begin{document}
	\maketitle

\begin{abstract}
Deep neural networks as image priors have been recently introduced for problems such as denoising, super-resolution and inpainting with promising performance gains over hand-crafted image priors such as sparsity. Unlike \textit{learned} generative priors they do not require any training over large datasets.  However, few theoretical guarantees exist in the scope of using untrained network priors for inverse imaging problems.  
We explore new applications and theory for using untrained neural network priors.  
Specifically, we consider the problem of solving linear inverse problems, such as compressive sensing, as well as non-linear problems, such as compressive phase retrieval. We model images to lie in the range of an untrained deep generative network with a fixed seed. We further present a projected gradient descent scheme that can be used for both compressive sensing and phase retrieval and provide rigorous theoretical guarantees for its convergence. We also show both theoretically as well as empirically that with deep neural network priors, one can achieve better compression rates for the same image quality as compared to when hand crafted priors are used. 


\end{abstract}

\section{Introduction}

\subsection{Motivation}

Deep neural networks have led to unprecedented success in solving several problems, specifically in the domain of inverse imaging. Image denoising \cite{vincent2010stacked}, super-resolution \cite{srcnn}, inpainting, compressed sensing \cite{onenet}, and phase retrieval \cite{prdeep} are among the many imaging applications that have benefited from the use of deep convolutional networks (CNNs) trained with thousands of images.

Apart from supervised learning, deep CNN models have also been used in unsupervised setups, for example Generative Adversarial Networks (GANs). Here, image priors based on a generative model \cite{gan} are learned from training data. In this context, neural networks emulate the probability distribution of the data inputs. GANs have been used to model signal prior by learning the distribution of training data. Such learned priors have replaced hand-crafted priors with high success rates \cite{onenet,CSGAN,PRGAN,deepcs}. 

However, the main challenge with these approaches is the requirement of massive amounts of training data. For instance, super-resolution CNN \cite{srcnn} uses ImageNet which contains millions of images. Moreover, convergence guarantees for training such networks are limited \cite{PRGAN}. 

In contrast, there has been recent interest in using \emph{untrained} neural networks as an image prior. Deep Image Prior \cite{DIP} and variants such as Deep Decoder \cite{DD} are capable of solving linear inverse imaging problems with no training data whatsover, while merely imposing an auto-encoder \cite{DIP} and decoder \cite{DD} architecture as a structural prior. For denoising, inpainting and super-resolution, deep image priors have shown superior reconstruction performance as compared to conventional methodologies such as basis pursuit denoising (BPDN) \cite{bpdn}, BM3D \cite{bm3d} as well as convolutional sparse coding \cite{CSC}. Similar emperical results have been claimed very recently in the context of time-series data for audio applications \cite{timeseriesdnp,audiodnp}. The theme in all of these approaches is the same: to design a prior that exploits \textit{local} image correlation, instead of global statistics, and find a good low-dimensional \emph{neural} representation of natural images. However, most of these works have very limited \cite{MLCSC,DD} or no theoretical guarantees.

Neural networks priors for compressive imaging has only recently been explored. In the context of compressive sensing (CS), \cite{csdip} uses Deep Image Prior along with \textit{learned regularization} for reconstructing images from compressive measurements \cite{cs}. However, the model described still relies on training data for learning appropriate regularization parameters. For the problem of compressive sensing, priors such as sparsity \cite{cosamp} and structured sparsity \cite{modelcs} have been traditionally used. 

Phase retrieval is another inverse imaging problem in several Fourier imaging applications, which involves reconstructing images from magnitude-only measurements. Compressive phase retrieval (CPR) models use sparse priors for reducing sample requirements; however, standard techniques from recent literature \cite{copram} suggest a quadratic dependence of number of measurements on the sparsity level for recovering sparse images from magnitude-only Gaussian measurements and the design of a smart initialization scheme \cite{sparta,copram}. If a prior is learned via a GAN \cite{PRGAN}, \cite{ptychogan}, then this requirement can be brought down; however one requires sufficient training data, which can be prohibitively expensive to obtain in domains such as medical or astronomical imaging.

\subsection{Our contributions}
In this paper, we explore, in depth, the use of untrained deep neural networks as an image prior for inverting images from under-sampled linear and non-linear measurements. Specifically, we assume that the image, ${x^*}^{d\times 1}$ has $d$ pixels. We further assume that the image $x^*$ belongs to the range spanned by the weights of a deep \textit{under-parameterized} untrained  neural network $G(\w;z)$, which we denote by $\S$, where $\w$ is a set of the weights of the deep network and $z$ is the latent code. 
The compressive measurements are stored in vector $y =  f(x^*)$, where $f$ embeds either compressive linear (defined by operator $A(\cdot)$) or compressive magnitude-only (defined by operator $|A(\cdot)|$) measurements. The task is to reconstruct image $\hat{x}$ which corresponds to small measurement error $\min_{x\in\S} \|f(x) - y\|_2^2$. 
With this setup, we establish theoretical guarantees for successful image reconstruction from both measurement schemes under untrained network priors. 

Our specific contributions are as follows:
\begin{itemize}[leftmargin=*,topsep=0pt, partopsep=0pt]
	\item We first present a new variant of the Restricted Eigenvalue Condition (REC) \cite{cs} via a covering number argument for the range of images $\S$ spanned by a deep untrained neural network. We use this result to guarantee unique image reconstruction for two different compressive imaging schemes.
	
	
	\item We propose a projected gradient descent (PGD) algorithm for solving the problem of compressive sensing with a deep untrained network prior. To our knowledge this is the first paper to use deep neural network priors for compressive sensing \footnote{We note recent concurrent work in \cite{csdd} which explores a similar approach for compressive sensing; however our paper focuses theoretical guarantees rooted in an algorithmic procedure.}, which relies on no training data\footnote{\cite{csdip} requires training data for learning a regularization function.}. We analyze the conditions under which PGD provably converges and report the sample complexity requirements corresponding to it. We also show superior performance of this framework via empirical results.
	
	\item We are the first to use deep network priors in the context of phase retrieval. We introduce a novel formulation, to solve compressive phase retrieval with fewer measurements as compared to state-of-art.  We further provide preliminary guarantees for the convergence of a projected gradient descent scheme to solve the problem of compressive phase retrieval. We empirically show significant improvements in image reconstruction quality as compared to prior works.  
\end{itemize}
We note that our sample complexity results rely on the number of parameters of the assumed deep network prior. Therefore, to get meaningful bounds, our network priors are \textit{under-parameterized}, in that the total number of unknown parameters of the deep network is smaller than the dimension of the image. To ensure this, we build upon the formulation of the deep decoder \cite{DD}, which is a special network architecture resembling the decoder of an autoencoder (or generator of a GAN). The requirement of under-parameterization of deep network priors is natural; the goal is to design priors that \textit{concisely} represent natural images. Moreover, this also ensures that the network does not fit noise \cite{DD}. Due to these merits, we use select the deep decoder architecture for all analyses in this paper.

\subsection{Prior work}

Sparsifying transforms have long been used to constrain the solutions of inverse imaging problems in the context of denoising or inpainting. Conventional approaches to solve these problems include Basis Pursuit Denoising (BPDN) or Lasso \cite{bpdn}, TVAL3 \cite{tval3}, which rely on using $\ell_0$, $\ell_1$ and total variation (TV) regularizations on the image to be recovered. Sparsity based priors are highly effective and dataset independent, however it heavily relies on choosing a good sparsifying basis \cite{wavelet}. 

Instead of hand-picking the sparsifying transform, in dictionary learning one learns both the sparsifying transform and the sparse code \cite{DL}. The dictionary captures global statistics of a given dataset \footnote{Local structural information from a single image can also be used to learn dictionaries, by constructing several overlapping crops or patches of a single image.}. Multi-layer convolutional sparse coding \cite{MLCSC} is an extension of sparse coding which models a given dataset in the form of a product of several linear dictionaries, all of which are convolutional in nature and this problem is challenging. 

Generative adversarial networks (GAN) \cite{gan} have been used to generate photo-realistic images in an unsupervised fashion. The generator consists of stacked convolutions and maps random low-dimensional noise vectors to full sized images. 
GAN priors have been successfully used for inverse imaging problems \cite{CSGAN,PRGAN,hyder2019alternating,shah2018solving,deepcs}. The shortcomings of this approach are two-fold: test images are strictly restricted to the range of a trained generator, and the requirement of sufficient training data.

Sparse signal recovery from linear compressive measurements \cite{cs} as well as magnitude-only compressive measurements \cite{copram} has been extensively studied, with several algorithmic approaches \cite{cosamp,copram}. 
In all of these approaches, modeling the low-dimensional embedding is challenging and may not be captured correctly using simple hand-crafted priors such as structured sparsity \cite{modelcs}. Since it is hard to estimate these hyper-parameters accurately, the number of samples required to reconstruct the image is often much higher than information theoretic limits \cite{jagatap2019sample,CSGAN}. 

The problem of compressive phase retrieval specifically, is even more challenging because it is non-convex. Several papers in recent literature \cite{twf,thwf,copram} rely on the design of a spectral initialization scheme which ensures that one can subsequently optimize over a convex ball of the problem. However this initialization requirement results in high sample requirements and is a bottleneck in achieving information theoretically optimal sample complexity. 

Deep image prior \cite{DIP} (DIP) uses primarily an encoder-decoder as a \textit{prior} on the image, alongside an early stopping condition, for inverse imaging problems such as denoising, super-resolution and inpainting.  Deep decoder \cite{DD} (DD) improves upon DIP, providing a much simpler, \textit{underparameterized} architecture, to learn a low-dimensional manifold (latent code) and a decoding operation from this latent code to the full image. 
Because it is under parameterized, deep decoder does not fit noise, and therefore does not require early stopping.

Deep network priors in the context of compressive imaging have only recently been explored \cite{csdip}, and only in the context of compressive sensing. In contrast with \cite{csdip} which extends the idea of a Deep Image Prior to incorporate learned regularizations, in this paper we focus more on theoretical aspects of the problem and also explore applications in compressive phase retrieval. To our knowledge the application of deep network priors to compressive phase retrieval is novel.

\section{Notation} \label{sec:notation}

Throughout the paper, lower case letters denote vectors, such as $v$ and upper case letters for matrices, such as $M$. A set of variables subscripted with different indices is represented with bold-faced shorthand of the following form: $\w := \{W_1,W_2,\dots W_L\}$. The neural network consists of $L$ layers, each layer denoted as $W_l$, with $l\in\{1,\dots L\}$ and are $1\times 1$ convolutional. 
Up-sampling operators are denoted by $U_l$. Vectorization of a matrix is written as vec$(\cdot)$. The activation function considered is Rectified Linear Unit (ReLU), denoted as $\sigma(\cdot)$. Hadamard or element-wise product is denoted by $\circ$. Element-wise absolute valued vector is denoted by $|v|$. Unless mentioned otherwise, $\|v\|$ denotes vector $\ell_2$-norm and $\|M\|$ denotes spectral norm $\|M\|_2$.

\section{Problem setup}

\subsection{Deep neural network priors} \label{subsec:priors}
In this paper we discuss the problem of inverting a mapping $x\to y$ of the form:
\begin{align*}
y = f(x) 
\end{align*}
where $x = \text{vec}(X)^{dk}$ is a $d$-dimensional signal $X^{d\times k}$ (vectorized image), with $k$  channels and $f: x\to y \in \real^n$ captures a compressive measurement procedure, such as a linear operator $A(\cdot)$ or magnitude only measurements $|A(\cdot)|$ and $n < dk$. We elaborate further on the exact structure of $f$ in the next subsection (Section \ref{subsec:models}). The task of reconstructing image $x$ from measurements $y$ can be formulated as an optimization problem of the form:
\begin{align} \label{eq:min}
\min_{x\in \S} \|y - f(x)\|^2_2 
\end{align}
where we have chosen the $\ell_2$-squared loss function and where $\S$ captures the prior on the image. 

If the image $x$ can be represented as the action of a deep generative network $G(\w;z)$ with weights $\w$ on some latent code $z$, such that $x=G(\w;z)$, then the set $\S$ captures the characteristics of $G(\w;z)$. 
The latent code $z := \text{vec}(Z_1)\:\text{with} \: Z_1 \in \real^{d_1\times k_1}$ is a low-dimensional embedding with dimension $d_1 k_1 \ll d k$ and its elements are generated from uniform random distribution.

When the network $G(\cdot)$ and its weights $\w:=\{W_1,\dots W_L\}$ are \textit{known} (from pre-training a generative network over large datasets) and fixed, 
%
the task is to obtain an estimate $\hat{x} = G(\w;\hat{z})$, which indirectly translates to finding the optimal latent space encoding $\hat{z}$ .
This problem has been studied in \cite{CSGAN,PRGAN} in the form of using learned GAN priors for inverse imaging.

In this paper however, the weights of the generator $\w$ are \textit{not pre-trained}; rather, the task is to estimate image $\hat{x} = G(\hat{\w};z) \approx G(\w^*;z) = x^*$ and corresponding weights $\hat{\w}$, for a \textit{fixed} seed $z$, where $x^*$ is assumed to be the true image and the true weights $\w^*$ (possibly non-unique) satisfy $\w^* = \min_{\w}\|x^*-G(\w;z)\|_2^2$. Note that the optimization in Eq. \ref{eq:min} is equivalent to substituting the surjective mapping $G:\w \to x$, and optimizing over $\w$, 
\begin{align} \label{eq:min2}
\min_{\w}\|y - f(G(\w;z))\|^2_2 ,
\end{align}
and estimate weights $\hat{\w}$ and corresponding image $\hat{x}$.
 
Specifically, the untrained network $G(\w;z)$ takes the form of an expansive neural network; a decoder architecture similar to the one in \cite{DD} \footnote{Alternatively, one may assume the architecture of the generator of a DCGAN \cite{DCGAN,csdip}.}. The neural network is composed of $L$ weight layers $W_l$, indexed by $l \in \{1,\dots, L\}$ and are $1\times 1$ convolutions, upsampling operators $U_l$ for $l\in \{1,\dots L-1\}$ and ReLU activation $\sigma(\cdot)$ and is expressed as follows 
\begin{align} \label{eq:decoder}
x = G(\w;z) = U_{L-1}\sigma(Z_{L-1} W_{L-1}) W_L = Z_L W_{L},
\end{align}
where $\sigma(\cdot)$ represents the action of ReLU operation, $Z_i^{d_{i}\times k_{i}} = U_{i-1}\sigma(Z_{i-1}W_{i-1})$, for $i=2,\dots L$, $z = \text{vec}(Z_1)$, 
$d_L=d$ and $W_L \in \reals^{k_L\times k}$. 

To capture the range of images spanned by the deep neural network architecture described above, we formally introduce the main assumption in our paper through Definition \ref{def:set}. Without loss in generality, we set $k=1$ for the rest of this paper, while noting that the techniques carry over to general $k$.

\begin{definition} \label{def:set}
	A given image $x \in \reals^d$ is said to obey an untrained neural network prior if it belongs to a set $\S$ defined as:  
	$$\S := \{x| x = G(\w;z)\}$$
	where $z$ is a (randomly chosen, fixed) latent code vector and $G(\w;z)$ has the form in Eq.\ \ref{eq:decoder}.
\end{definition}

\subsection{Observation models and assumptions} \label{subsec:models}

We now discuss the compressive measurement setup in more detail. Compressive measurement schemes were developed in \cite{cs} for efficient imaging and storage of images and work only as long as certain structural assumptions on the signal (or image) are met. The optimization problem in Eq.\ref{eq:min} is
 non-convex in general, partly dictated by the non-convexity of set $\S$. Moreover, in the case of phase retrieval, the loss function is itself non-convex. Therefore unique signal recovery for either problems is not guaranteed without making specific assumptions on the measurement setup. 

In this paper, we assume that the measurement operation can be represented by the action of a Gaussian matrix $A$ which is rank-deficient ($n<d$). The entries of this matrix are such that $A_{ij} \sim \mathcal{N}(0,1/n)$. Linear compressive measurements take the form $y=Ax$ and magnitude-only measurements take the form $y=|Ax|$. We formally discuss the two different imaging schemes in the next two sections. We also present algorithms and theoretical guarantees for their convergence. For both algorithms, we require that a special $(S,\gamma,\beta)$-REC (Restricted Eigenvalue Condition) holds for measurement matrix $A$, which is defined below.

\begin{definition} \label{def:RIP}
	$(\S,\gamma, \beta )$-REC: Set-Restricted Eigenvalue Condition with parameters $\gamma,\beta$: 
	
	For parameters $\gamma, \beta >0$, a matrix $A \in \mathbb{R}^{n\times d}$ satisfies $(\S,\gamma,\beta)$-REC, if for all $x \in \S$, 
	$$ \gamma\|x\|^2 \leq \|Ax\|^2 \leq \beta\|x\|^2 .$$
	We refer to the left (lower) inequality as $(\S,\gamma)$-REC and right (upper) inequality as $(\S,\beta)$-REC.
\end{definition}

The $(\S,1-\alpha,1+\alpha)$ REC is achieved by Gaussian matrix $A$ under certain assumptions, which we state and prove via Lemma \ref{lem:leftRIP} as follows. 

\begin{restatable}{lemma}{lemRIP} \label{lem:leftRIP}
	If an image $x\in \mathbb{R}^{d}$ has a decoder prior (captured in set $\S$), where the decoder consists of weights $\w$, has two layers, and piece-wise linear activation (ReLU), a random Gaussian matrix $A \in \mathbb{R}^{n\times d}$ with elements from $\mathcal{N}(0,1/n)$, satisfies  $(\S,1-\alpha,1+\alpha)$-REC, with probability $1-e^{-c\alpha^2 n}$, as long as $n = O\left(\frac{1}{\alpha^2} k_1 k_2 \log d\right)$, for small constant $c$ and $0<\alpha<1$.
\end{restatable}

\textit{\textbf{Proof sketch:}} We use a union of sub-spaces model, similar to that developed in \cite{CSGAN} which was developed for GAN priors, to capture the range of a deep untrained network. 

Our method uses a \emph{linearization principle}. If the output sign of any ReLU activation $\sigma(\cdot)$ on its inputs were known \emph{a priori}, then the mapping $x=G(\w;z)$ becomes a product of linear weight matrices and linear upsampling operators acting on the latent code $z$. 
The bulk of the proof relies on constructing a counting argument for the number of such {linearized networks}; call that number $N$. For a fixed linear subspace, the image $x$ has a representation of the form $x = U Z' w $, where $Z'$ absorbs the effect of all ReLU operations. The latent code $Z$ is fixed and known and $w$ represents a stack of the trainable weights with $w$ which is $k_1 k_2$-dimensional. An oblivious subspace embedding (OSE) of $x$ takes the form
$$
(1-\alpha)\|x\|^2 \leq \|Ax\|^2 \leq (1+\alpha)\|x\|^2,
$$
where $A$ is a Gaussian matrix, and holds for all $ k_1 k_2$-dimensional vectors $w$, with high probability as long as $n=O( k_1k_2/\alpha^2)$. We further require to take a union bound over all possible such {linearized} networks, which is upper bounded as $N \leq d^{k_1k_2}$. The sample complexity corresponding to this bound is then computed to complete the set-REC result. The complete proof can be found in Appendix \ref{sec:appendix} and a discussion on the sample complexity is presented in Appendix \ref{sec:discussion}.

\section{Linear compressive sensing with deep image prior} \label{sec:CS}
We now analyze linear compressed Gaussian measurements of a vectorized image $x$, with a deep network prior. The reconstruction problem assumes the following form:
\begin{align} \label{eq:2layer}
&\min_{x} \quad \| y - A x \|^2\quad\text{s.t.}\quad x = G(\w;z) ,
\end{align}
where $A \in \mathbb{R}^{n \times d}$ is Gaussian matrix with $n < d$, unknown weight matrices $\w$ and latent code $z$ which is fixed. We solve this problem via Algorithm \ref{algo:pgd}, Network Projected Gradient Descent (Net-PGD) for compressed sensing recovery.

Specifically, we break down the minimization into two parts; we first solve an unconstrained loss minimization of the objective function in Eq. \ref{eq:2layer} by implementing one step of gradient descent in Step 3 of Algorithm \ref{algo:pgd}. The update $v^t$ typically does not adhere to the deep network prior constraint $v^t \not\in \S$. To ensure that this happens, we solve a projection step in Line 4 of Algorithm \ref{algo:pgd}, which happens to be the same as fitting a deep network prior to a noisy image. We iterate through this procedure in an alternating fashion until the estimates $x^t$ converge to $x^*$ within error factor $\epsilon$.

We further establish convergence guarantees for Algorithm \ref{algo:pgd} in Theorem \ref{thm:cs-converge}.
\begin{algorithm}[t]
	\caption{Net-PGD for compressed sensing recovery.}
	\label{algo:pgd}
	\begin{algorithmic}[1]
		
		\STATE \textbf{Input:} $y, A, z = \text{vec}(Z_1), \eta, T=\log \frac{1}{\epsilon}$
		\FOR{$t = 1, \cdots,T$}
		\STATE	$v^{t}\: \leftarrow x^t - \eta A^\top (Ax^t-y)$ \hspace{0.99cm}\COMMENT{gradient step for least squares}
		\STATE	$\w^{t} \leftarrow \arg\min\limits_{\w}\|v^t - G(\w;z)\|$ \hspace{0.4cm} \COMMENT{projection to range of deep network}
		\STATE $x^{t+1} \: \leftarrow G(\w^t;z)$
		\ENDFOR
		\STATE \textbf{Output} $ \hat{x} \leftarrow {x}^{T}$.
	\end{algorithmic}
\end{algorithm}

\begin{restatable}{theorem}{thmCS} \label{thm:cs-converge}
	
	Suppose the sampling matrix $A^{n\times d}$ satisfies $(\S,1-\alpha,1+\alpha)$-REC with high probability then, Algorithm \ref{algo:pgd}, with $\eta$ small enough, produces $\hat{x}$ such that $\|\hat{x}-x^*\| \leq \epsilon$ and requires $T \propto \log \frac{1}{\epsilon}$ iterations.
\end{restatable}

\textbf{\textit{Proof sketch:}} The proof of this theorem predominantly relies on our new set-restricted REC result and uses standard techniques from compressed sensing theory. Indicating the loss function in Eq. \ref{eq:2layer} as $L(x^t) = \|y-Ax^t\|^2$, we aim to establish a contraction of the form $L(x^{t+1}) < \nu L(x^t)$, with $\nu < 1$. To achieve this, we combine the projection criterion in Step 4 of Algorithm \ref{algo:pgd}, which strictly implies that 
$$
\|x^{t+1}-v^t\| \leq \|x^*-v^t\|
$$
and $v^t = x^t - \eta A^\top (Ax^t - y)$ from Step 3 of Algorithm \ref{algo:pgd}, where $\eta$ is chosen appropriately. Therefore,
\[
\|x^{t+1}-x^t+\eta A^\top A(x^t-x^*)\|^2\leq \|x^{*}-x^t+\eta A^\top A(x^t-x^*)\|^2.
\]
Furthermore, we utilize $(\S,1-\alpha,1+\alpha)$-REC and its Corollary \ref{cor:diff} (refer Appendix \ref{sec:appendix}) which apply to $x^*,x^t,x^{t+1} \in \S$, to show that
\[
L(x^{t+1}) \leq \nu L(x^t) 
\]
and subsequently the error contraction $\|x^{t+1}-x^*\|\leq \nu_o\|x^t-x^*\|$, with $\nu,\nu_o<1$ to guarantee linear convergence of Net-PGD for compressed sensing recovery. This convergence result implies that Net-PGD requires $T\propto \log 1/\epsilon$ iterations to produce $\hat{x}$ within $\epsilon$-accuracy of $x^*$. The complete proof of Theorem \ref{thm:cs-converge} can be found in Appendix \ref{sec:appendix}. In Appendix \ref{sec:projection} we provide some exposition on the projection step (line 4 of Algorithm \ref{algo:pgd}).

\section{Compressive phase retrieval under deep image prior} \label{sec:PR}
In compressive phase retrieval, one wants to reconstruct a signal $x\approx x^* \in \S$ from measurements of the form $y=|Ax^*|$ and therefore the objective is to minimize the following  
	\begin{align} \label{eq:2pr}
&\min_{x} \quad \| y - |A x| \|^2 \quad \text{s.t.}\quad x = G(\w;z), 
\end{align}
	where $n<d$ and $A$ is Gaussian, $z$ is a fixed seed and weights $\w$ need to be estimated. 
We propose a Network Projected Gradient Descent (Net-PGD) for compressive phase retrieval to solve this problem, which is presented in Algorithm \ref{algo:pgdpr}. 
	
Algorithm \ref{algo:pgdpr} broadly consists of two parts. For the first part, in Line 3 we estimate the phase of the current estimate and in Line 4 we use this to compute the Wirtinger gradient \cite{twf} and execute one step for solving an unconstrained phase retrieval problem with gradient descent. The second part of the algorithm is (Line 5), estimating the weights of the deep network prior with noisy input $v^t$. This is the projection step and ensures that the output $\w^t$ and subsequently the image estimate $x^t = G(\w^t;z)$ lies in the range of the decoder $G(\cdot)$ outlined by set $\S$.
	
	\begin{algorithm}[t]
		\caption{Net-PGD for compressive phase retrieval.}
		
		\label{algo:pgdpr}
		\begin{algorithmic}[1]
			
			\STATE \textbf{Input:} $A, z = \text{vec}(Z_1), \eta, T=\log\frac{1}{\epsilon}, x^0$ s.t. $\|x^0-x^*\|\leq \delta_i\|x^*\|$.
			\FOR{$t = 1, \cdots,T$}
			\STATE $p^t \leftarrow \text{sign}(Ax^t)$ \hspace{3cm}\COMMENT{phase estimation}
			\STATE	$v^{t} \leftarrow x^t - \eta A^\top (Ax^t-y\circ p^t)$ \hspace{0.8cm} \COMMENT{gradient step for phase retrieval}
			\STATE	$\w^{t} \leftarrow \arg\min\limits_{\w}\|v^t - G(\w;z)\|$ \hspace{0.86cm} \COMMENT{projection to range of deep network}
			\STATE $x^{t+1} \leftarrow G(\w^t;z)$
			\ENDFOR
			\STATE \textbf{Output} $ \hat{x}\leftarrow {x}^{T}$.
		\end{algorithmic}
	\end{algorithm}

We highlight that the problem in Eq. \ref{eq:2pr} is significantly more challenging than the one in Eq. \ref{eq:2layer}. The difficulty hinges on estimating the missing phase information accurately. For a real-valued vectors, there are $2^n$ different phase vectors $p=\text{sign}(Ax)$ for a fixed choice of $x$, which satisfy $y=|Ax|$, moreover the entries of $p$ are restricted to $\{1,-1\}$. Hence, phase estimation is a non-convex problem. Therefore, with Algorithm \ref{algo:pgdpr} the problem in Eq.\ref{eq:2pr} can only be solved to convergence locally; an initialization scheme is required to establish global convergence guarantees. We highlight the guarantees of Algorithm \ref{algo:pgdpr} in Theorem \ref{thm:pr-converge}.
	
	\begin{restatable}{theorem}{thmPR} \label{thm:pr-converge}
		
		Suppose the sampling matrix $A^{n\times d}$ with Gaussian entries satisfies $(\S,1-\alpha,1+\alpha)$-REC with high probability, Algorithm \ref{algo:pgdpr} solves Eq. \ref{eq:2pr} with $\eta$ small enough, such that $\|\hat{x}-x^*\| \leq \epsilon$, as long as the weights of the two-layer decoder are initialized appropriately and the number of  measurements is $n = O\left( k_1k_2 \log d\right)$.
	\end{restatable}

\textit{\textbf{Proof sketch:}}
The proof for Theorem \ref{thm:pr-converge} relies on two important results; $(\S,1-\alpha,1+\alpha)$-REC and  Lemma \ref{lem:phaseerr} which establishes a bound on the phase estimation error. Formally, the update in Step 4 of Algorithm \ref{algo:pgdpr} can be re-written as
\begin{align*}
v^{t+1} &= x^t - \eta A^\top \left(Ax^t - Ax^* \circ \text{sign}(Ax^*) \circ\text{sign}(Ax^t)\right) = x^t - \eta A^\top \left(Ax^t - Ax^*\right) - \eta \varepsilon_p^t
\end{align*}
where $\varepsilon_p^t :=  A^\top  A x^*\circ(1- \text{sign}(Ax^*) \circ\text{sign}(Ax^t))$ is \textit{phase estimation} error.

If $\text{sign}(Ax^*) \approx \text{sign}(Ax^t)$, then the above resembles the gradient step from the linear compressive sensing formulation. Thus, if $x^0$ is initialized well, the error due to phase mis-match $\varepsilon_p^t$ can be bounded, and subsequently, a convergence result can be formulated.

Next, Step 5 of Algorithm \ref{algo:pgdpr} learns weights $\w^t$ that produce $x^t = G(\w^t;z)$, such that
\begin{align*}
\|x^{t+1}-v^t\| &\leq \|x^t-v^t\| 
\end{align*}
for $t=\{1,2,\dots T\}$. Then, the above projection rule yields:
$$
\|x^{t+1} - v^{t+1} + v^{t+1} -  x^* \| \leq \|x^{t+1} - v^{t+1}\| + \|x^* - v^{t+1}\| \leq 2 \|x^*-v^{t+1}\|,
$$
Using the update rule from Eq. \ref{eq:prstep} and plugging in for $v^{t+1}$:
\begin{align*}
\frac{1}{2}\|x^{t+1}-x^*\| &\leq \|(1-\eta A^\top A)h^t\| + \|\varepsilon_p^t\|
\end{align*}
where $\eta$ is chosen appropriately. The rest of the proof relies on bounding the first term via  matrix norm inequalities using Corollary \ref{cor:eig} (in Appendix \ref{sec:appendix}) of $(\S,1-\alpha,1+\alpha)$-REC as $\|(1-\eta A^\top A)h^t\| \leq  \rho_o \|h^t\|$
and the second term  is bounded via Lemma \ref{lem:phaseerr} as $\|\varepsilon_p^t\| \leq \delta_o \|x^t-x^*\|$ as long as $\|x^0-x^*\| \leq \delta_i \|x^*\|$. Hence we obtain a convergence criterion of the form
\begin{align*}
\|x^{t+1}-x^*\| \leq 2(\rho_o + \eta\delta_o)\|x^t - x^*\| := \rho \|x^t - x^*\|.
\end{align*}
where $\rho<1$. Note that this proof relies on a bound on the phase error $\|\varepsilon_p^t\|$ which is established via Lemma \ref{lem:phaseerr}. The complete proof for Theorem \ref{thm:pr-converge} can be found in Appendix \ref{sec:appendix}. In Appendix \ref{sec:projection} we provide some exposition on the projection step (line 5 of Algorithm \ref{algo:pgdpr}). In our experiments (Section \ref{sec:exps}) we note that a uniform random initialization of the weights $\w^0$ (which is common in training neural networks), to yield $x^0 = G(\w^0;z)$ is sufficient for Net-PGD to succeed for compressive phase retrieval. In  Appendix \ref{sec:addexp} we show experimental evidence to support this claim.

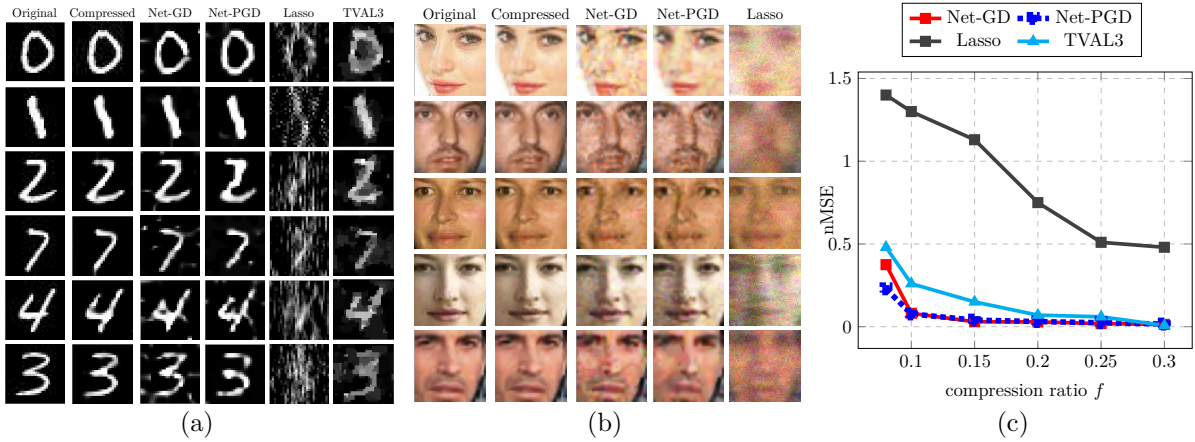
\begin{figure}[!t]
	\centering
	\begin{tabular}{ccc}
	\resizebox{0.32\textwidth}{!}{
		\begin{tabular}{cccccc}
	 Original & Compressed &  Net-GD & Net-PGD & Lasso &  TVAL3\\
			
			\includegraphics[width=0.1\textwidth]{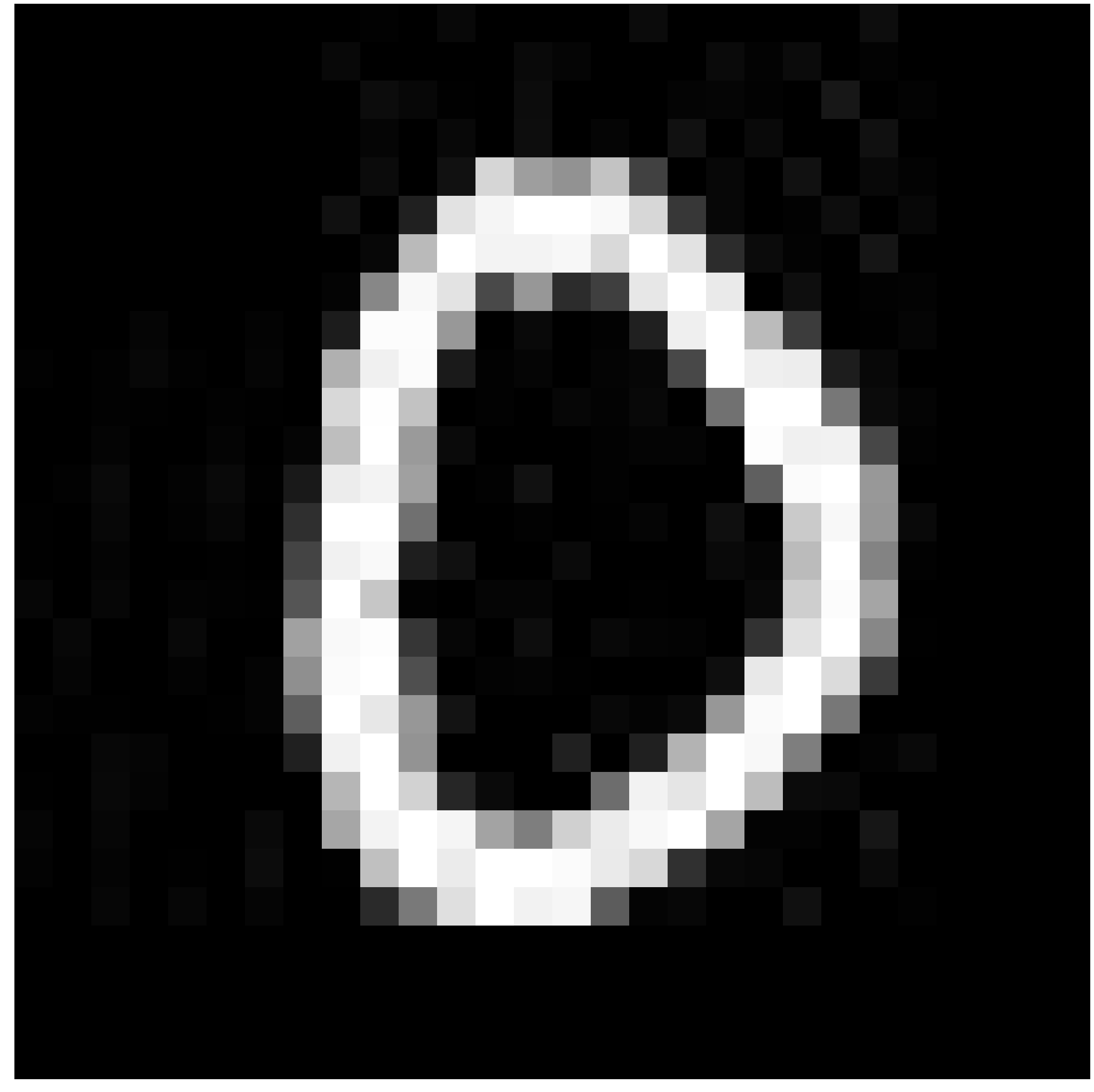} &
			\includegraphics[width=0.1\textwidth]{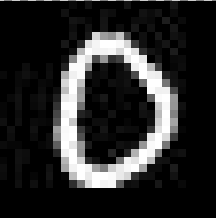}	&		
			\includegraphics[width=0.1\textwidth]{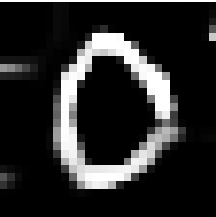} &
			\includegraphics[width=0.1\textwidth]{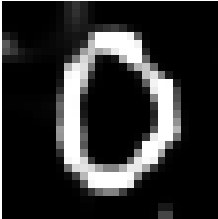} &
			\includegraphics[width=0.1\textwidth]{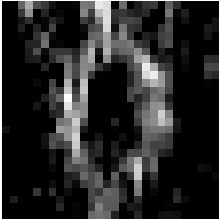} & 
			\includegraphics[width=0.1\textwidth]{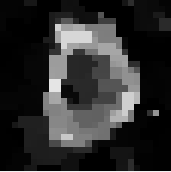}			
			\\
			\includegraphics[width=0.1\textwidth]{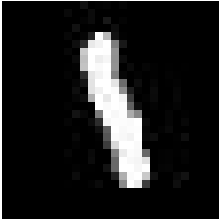} &
			\includegraphics[width=0.1\textwidth]{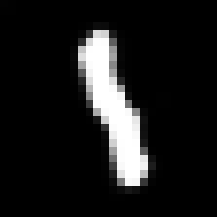}	&		
			\includegraphics[width=0.1\textwidth]{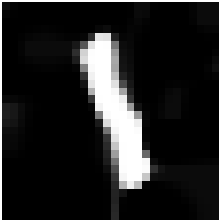} &
			\includegraphics[width=0.1\textwidth]{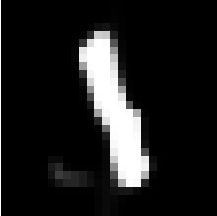} &
			\includegraphics[width=0.1\textwidth]{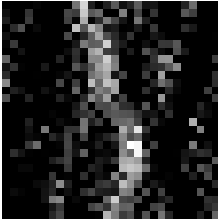} &
			\includegraphics[width=0.1\textwidth]{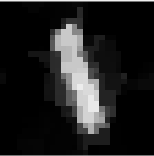} 			
			\\ 
			\includegraphics[width=0.1\textwidth]{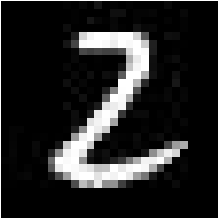} &
			\includegraphics[width=0.1\textwidth]{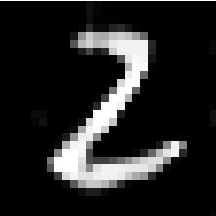}	&		
			\includegraphics[width=0.1\textwidth]{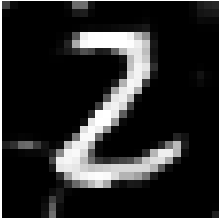} &
			\includegraphics[width=0.1\textwidth]{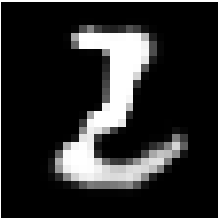} &
			\includegraphics[width=0.1\textwidth]{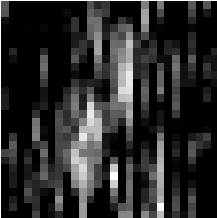} &
			\includegraphics[width=0.1\textwidth]{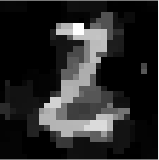}			
			\\
			
			\includegraphics[width=0.1\textwidth]{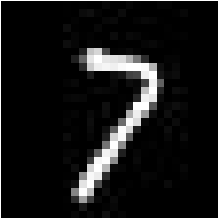} &
			\includegraphics[width=0.1\textwidth]{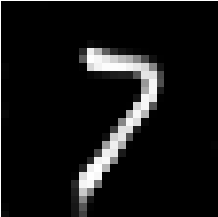}	&		
			\includegraphics[width=0.1\textwidth]{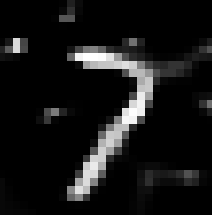} &
			\includegraphics[width=0.1\textwidth]{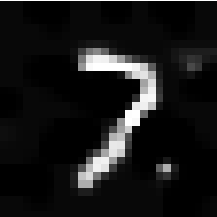} &
			\includegraphics[width=0.1\textwidth]{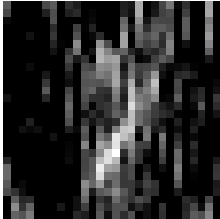} &
			\includegraphics[width=0.1\textwidth]{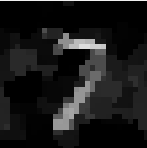} 
			\\
			\includegraphics[width=0.1\textwidth]{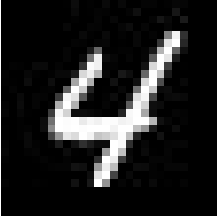} &
			\includegraphics[width=0.1\textwidth]{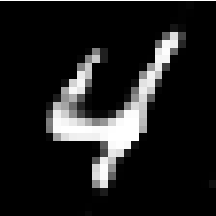}	&		
			\includegraphics[width=0.1\textwidth]{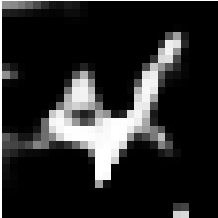} &
			\includegraphics[width=0.1\textwidth]{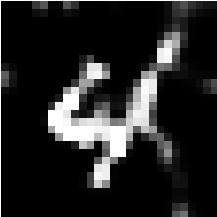} &
			\includegraphics[width=0.1\textwidth]{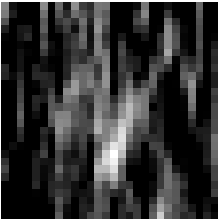} &
			\includegraphics[width=0.1\textwidth]{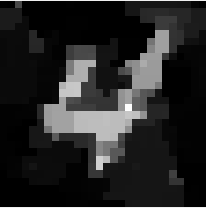} 			
			\\
			\includegraphics[width=0.1\textwidth]{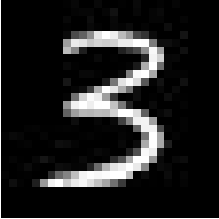} &
			\includegraphics[width=0.1\textwidth]{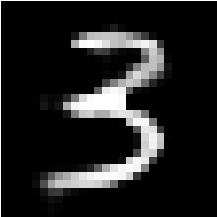}	&		
			\includegraphics[width=0.1\textwidth]{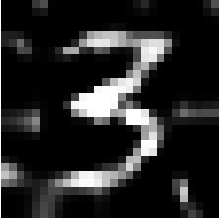} &
			\includegraphics[width=0.1\textwidth]{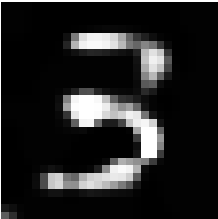} &
			\includegraphics[width=0.1\textwidth]{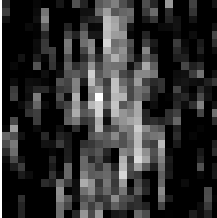} &
			\includegraphics[width=0.1\textwidth]{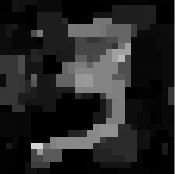}			
	\end{tabular}} & 	\resizebox{0.32\textwidth}{!}{
\begin{tabular}{ccccc}
Original &Compressed & Net-GD & Net-PGD & Lasso\\
\includegraphics[width=0.1\textwidth]{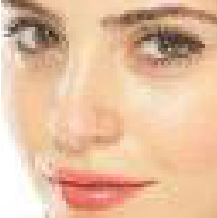} &
\includegraphics[width=0.1\textwidth]{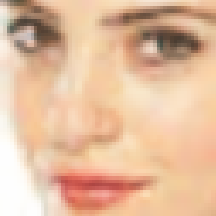}	&		
\includegraphics[width=0.1\textwidth]{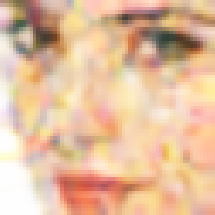} &
\includegraphics[width=0.1\textwidth]{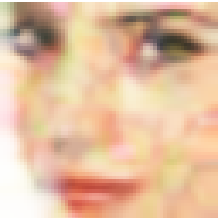} &
\includegraphics[width=0.1\textwidth]{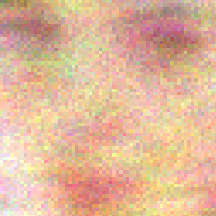} \\
\includegraphics[width=0.1\textwidth]{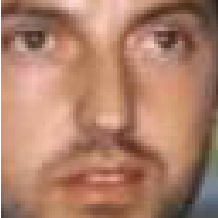} &
\includegraphics[width=0.1\textwidth]{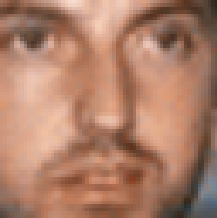}	&		
\includegraphics[width=0.1\textwidth]{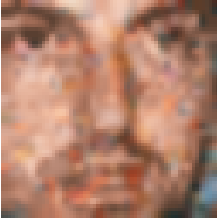} &
\includegraphics[width=0.1\textwidth]{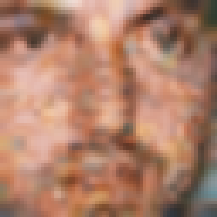} &
\includegraphics[width=0.1\textwidth]{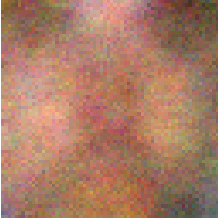} \\
\includegraphics[width=0.1\textwidth]{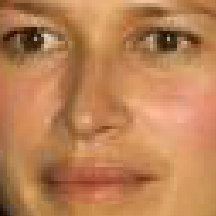} &
\includegraphics[width=0.1\textwidth]{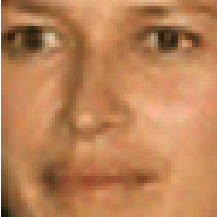}	&		
\includegraphics[width=0.1\textwidth]{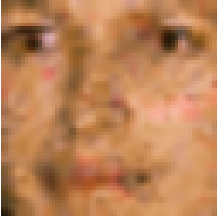} &
\includegraphics[width=0.1\textwidth]{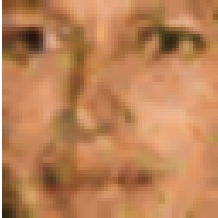} &
\includegraphics[width=0.1\textwidth]{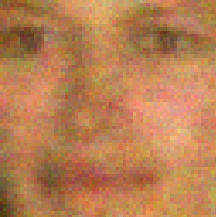} \\
\includegraphics[width=0.1\textwidth]{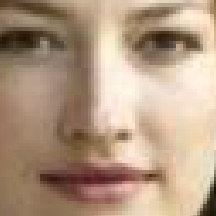} &
\includegraphics[width=0.1\textwidth]{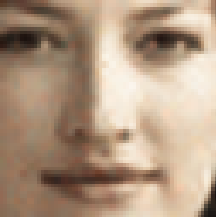}	&		
\includegraphics[width=0.1\textwidth]{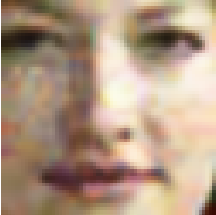} &
\includegraphics[width=0.1\textwidth]{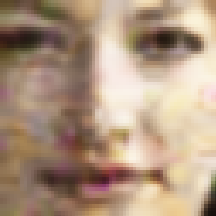} &
\includegraphics[width=0.1\textwidth]{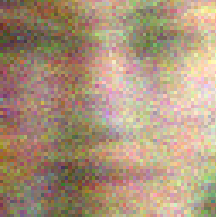} \\
\includegraphics[width=0.1\textwidth]{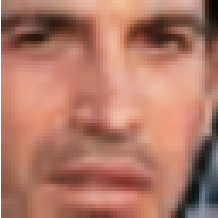} &
\includegraphics[width=0.1\textwidth]{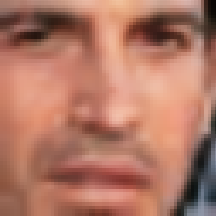}	&		
\includegraphics[width=0.1\textwidth]{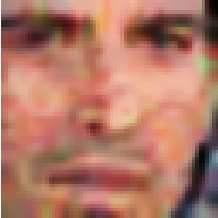} &
\includegraphics[width=0.1\textwidth]{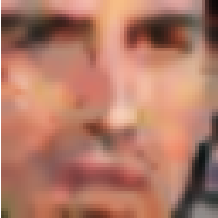} &
\includegraphics[width=0.1\textwidth]{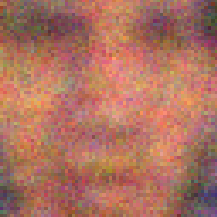} 
\end{tabular}} &
\begin{minipage}{0.32\textwidth}
\begin{tikzpicture}[scale=0.7]

\begin{axis}
[
width=1.5\textwidth,
xlabel= compression ratio $f$, 
ylabel= nMSE,
y label style={at={(0.1,0.5)}},
grid style = dashed,
grid=both,
legend style=
{at={(0.49,1.25)}, 
anchor= north , 
legend columns=2,
} ,
]

\addplot[color=red, solid,line width=2pt, mark size=2pt, mark=square*] plot coordinates {
(0.08,0.375)
(0.1,0.081)
(0.15,0.03)
(0.2,0.029)
(0.25,0.019)
(0.3,0.014)
};

\addplot[color=blue, dotted,line width=3pt, mark size=2pt, mark=square*] plot coordinates {
(0.08,0.229)
(0.1,0.077)
(0.15,0.04)
(0.2,0.031)
(0.25,0.024)
(0.3,0.016)
};

\addplot[color=darkgray, solid,line width=2pt, mark size=2pt, mark=square*] plot coordinates {
(0.08,1.4)
(0.1,1.3)
(0.15,1.13)
(0.2,0.75)
(0.25,0.51)
(0.3,0.48)
};

\addplot[color=cyan, solid,line width=2pt, mark size=2pt, mark=triangle*] plot coordinates {
(0.08,0.48)
(0.1,0.26)
(0.15,0.15)
(0.2,0.07)
(0.25,0.06)
(0.3,0.008)
};
\legend{Net-GD\\Net-PGD\\Lasso \\TVAL3\\}
\end{axis}
\hspace{-20pt}
\end{tikzpicture}
\end{minipage}
\\
(a) & (b) & (c)
	\end{tabular}
			
\normalfont
	\caption{ (CS) Reconstructed images from linear measurements (at compression rate $n/d=0.1$) with (a) $n=78$ measurements  for examples from MNIST, (b) $n=1228$ measurements for examples from CelebA, and (c) nMSE at different compression rates $f=n/d$ for fixed image from MNIST.} \label{fig:cs_img}

\end{figure}

\section{Experimental results} \label{sec:exps}
\textit{Dataset:} We use images from the MNIST database  and CelebA database to test our algorithms and reconstruct 6 grayscale (MNIST,  $28\times 28$ pixels $(d=784)$) and 5 RGB (CelebA) images. The CelebA dataset images are center cropped to size $64\times 64\times 3$ $(d=12288)$. The pixel values of all images are scaled to lie between 0 and 1.

\textit{Deep network architecture:}
We first optimize the deep network architecture which fit our example images such that $x^* \approx G(\w^*;z)$ (referred as ``compressed'' image). For MNIST images, the architecture was fixed to a 2 layer configuration $k_1=15,k_2=15,k_3 =10$, and for CelebA images, a 3 layer configuration with
$k_1=120,k_2=15,k_3=15,k_4=10$. Both architectures use bilinear upsampling operations. Further details on this setup can be found in Appendix \ref{sec:addexp}.

\textit{Measurement setup:}
We use a Gaussian measurement matrix of size $n\times d$ with $n$ varied such that (i) $n/d=0.08,0.1,0.15,0.2,0.25,0.3$ for compressive sensing and (ii)  $n/d=0.1,0.2,0.3,0.5,1,3$ for compressive phase retrieval. The elements of $A$ are picked such that $A_{i,j}\sim \mathcal{N}(0,1/n)$ and we report averaged reconstruction error values over $10$ different instantiations of $A$ for a fixed image (image of digit `0' from MNIST), network configuration and compression ratio $n/d$ .

\subsection{Compressive sensing}

\textit{Algorithms and baselines:}
We implement 4 schemes based on \textit{untrained} priors for solving CS, (i) gradient descent with deep network prior which solves Eq.\ref{eq:min2} (we call this Net-GD), similar to \cite{csdip} but without learned regularization (ii) Net-PGD, (iii) Lasso ($\ell_1$ regularization) with sparse prior in DCT basis and finally (iv) TVAL3 \cite{tval3} (Total Variation regularization). The TVAL3 code only works for grayscale images, therefore we do not use it for CelebA examples. The reconstructions are shown in Figure \ref{fig:cs_img} for images from (a) MNIST and (b) CelebA datasets. The implementation details can be found in Appendix \ref{sec:addexp}.

\begin{figure}
	\centering
	\begin{tabular}{ccc}
		\resizebox{0.28\textwidth}{!}{
			\begin{tabular}{ccccc}
				Original &Compressed & Net-GD & Net-PGD & Sparta\\ 
				\includegraphics[width=0.1\textwidth]{images/mnist/r0_img_0} &
				\includegraphics[width=0.1\textwidth]{images/mnist/rc_img_0}	&		
				\includegraphics[width=0.1\textwidth]{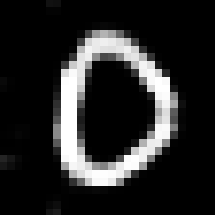} &
				\includegraphics[width=0.1\textwidth]{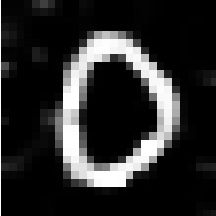} &
				\includegraphics[width=0.1\textwidth]{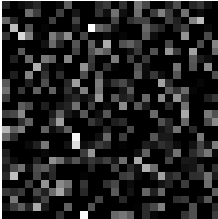} \\
				\includegraphics[width=0.1\textwidth]{images/mnist/r0_img_1} &
				\includegraphics[width=0.1\textwidth]{images/mnist/rc_img_1}	&		
				\includegraphics[width=0.1\textwidth]{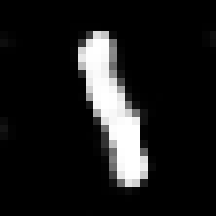} &
				\includegraphics[width=0.1\textwidth]{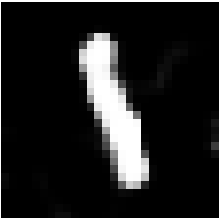} &
				\includegraphics[width=0.1\textwidth]{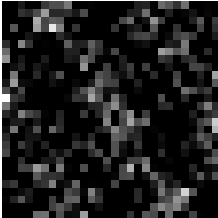} \\
				\includegraphics[width=0.1\textwidth]{images/mnist/r0_img_2} &
				\includegraphics[width=0.1\textwidth]{images/mnist/rc_img_2}	&		
				\includegraphics[width=0.1\textwidth]{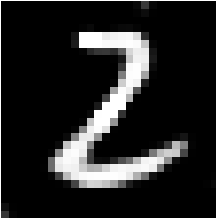} &
				\includegraphics[width=0.1\textwidth]{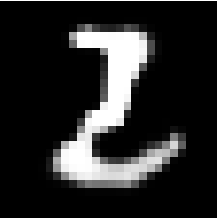} &
				\includegraphics[width=0.1\textwidth]{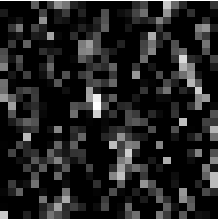} \\
				\includegraphics[width=0.1\textwidth]{images/mnist/r0_img_7} &
				\includegraphics[width=0.1\textwidth]{images/mnist/rc_img_7}	&		
				\includegraphics[width=0.1\textwidth]{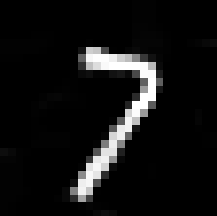} &
				\includegraphics[width=0.1\textwidth]{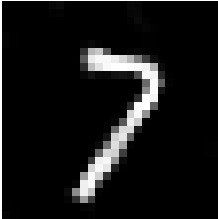} &
				\includegraphics[width=0.1\textwidth]{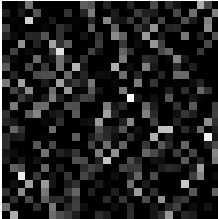} \\
				\includegraphics[width=0.1\textwidth]{images/mnist/r0_img_3} &
				\includegraphics[width=0.1\textwidth]{images/mnist/rc_img_3}	&		
				\includegraphics[width=0.1\textwidth]{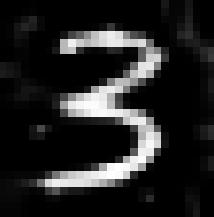} &
				\includegraphics[width=0.1\textwidth]{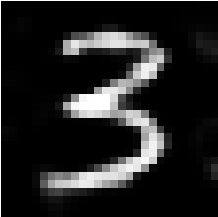} &
				\includegraphics[width=0.1\textwidth]{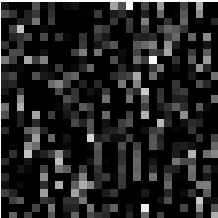} 
		\end{tabular}} & 
		
		\resizebox{0.3\textwidth}{!}{
			\begin{tabular}{cccccc}
				&& $n/d=0.5$ &&\\
				\includegraphics[width=0.1\textwidth]{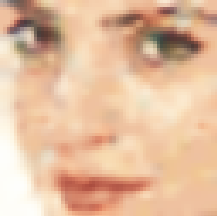} &
				\includegraphics[width=0.1\textwidth]{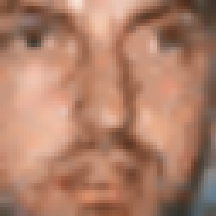} &
				\includegraphics[width=0.1\textwidth]{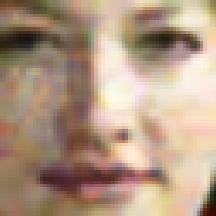} &
				\includegraphics[width=0.1\textwidth]{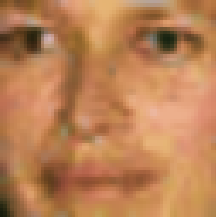}  &
				\includegraphics[width=0.1\textwidth]{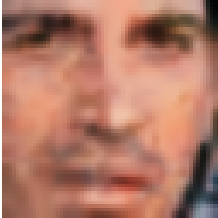} \\
				\includegraphics[width=0.1\textwidth]{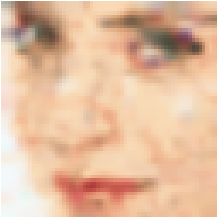} &	
				\includegraphics[width=0.1\textwidth]{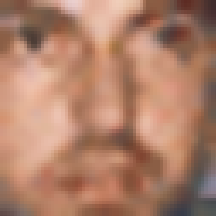}&	    						\includegraphics[width=0.1\textwidth]{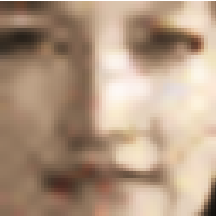} &
				\includegraphics[width=0.1\textwidth]{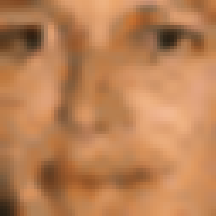} &
				\includegraphics[width=0.1\textwidth]{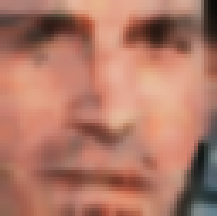}
				\\	
				&& $n/d=0.1$&&\\
				\includegraphics[width=0.1\textwidth]{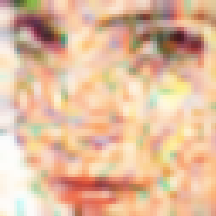} &
				\includegraphics[width=0.1\textwidth]{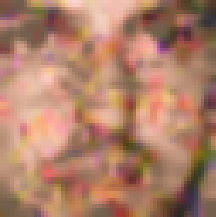} &	
				\includegraphics[width=0.1\textwidth]{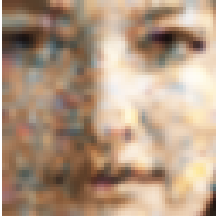} 	&		
				\includegraphics[width=0.1\textwidth]{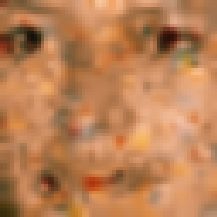} 	&	
				\includegraphics[width=0.1\textwidth]{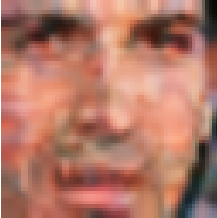} 	
				\\
				\includegraphics[width=0.1\textwidth]{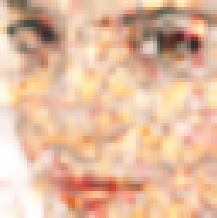} &
				\includegraphics[width=0.1\textwidth]{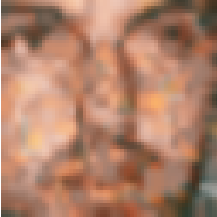} &	
				\includegraphics[width=0.1\textwidth]{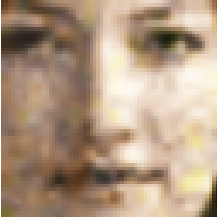} 	&	
				\includegraphics[width=0.1\textwidth]{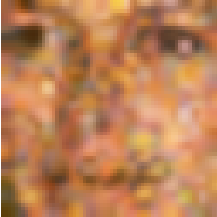} &				
				\includegraphics[width=0.1\textwidth]{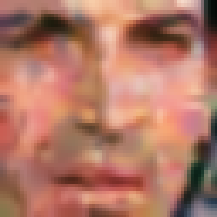} 
		\end{tabular}} &
		\begin{minipage}{0.32\textwidth}
			\centering
			\vspace{1cm}
			\begin{tikzpicture}[scale=0.8]
\begin{axis}
[
width=1.4\textwidth,
y label style={at={(0.1,0.5)}},
x label style={at={(0.5,0.05)}},
xlabel= compression ratio $f$, 
ylabel= nMSE,
grid style = dashed,
grid=both,
legend style=
{at={(0.45,0.6), font=\footnotesize}, 
anchor=south west, 
} ,
]

\addplot[color=red, solid,line width=2pt, mark size=2pt, mark=square*] plot coordinates {
(0.1,0.78)
(0.2,0.78)
(0.3,0.045)
(0.5,0.012)
(1,0.015)
(3,0.014)
};

\addplot[color=blue, dotted,line width=3pt, mark size=2pt, mark=square*] plot coordinates {
(0.1,0.85)
(0.2,0.79)
(0.3,0.07)
(0.5,0.017)
(1,0.017)
(3,0.016)
};

\addplot[color=darkgray, solid,line width=2pt, mark size=2pt, mark=square*] plot coordinates {
(0.3,0.89)
(0.5,0.75)
(1,0.41)
(3,0.004)
};

\legend{Net-GD\\Net-PGD\\Sparta \\}
\end{axis}
\end{tikzpicture}
		\end{minipage}\\
		(a) & (c) & (c)
	\end{tabular}	
	\caption{ (CPR) Reconstructed images from magnitude-only measurements (a) at compression rate of $n/d=0.3$ for MNIST, (b) at compression rates of $n/d=0.1,0.5$ for CelebA with (row 1,3) Net-GD and (row 2,4) Net-PGD, (c) nMSE at different compression rates $f=n/d$ for fixed image from MNIST.} \label{fig:pr_img}
\end{figure} 

\textit{Performance metrics:} We compare reconstruction quality using  normalized Mean-Squared Error (nMSE), which is calculated as $\|\hat{x} - x^*\|^2/\|x^*\|^2$. We plot the variation of the nMSE with different compression rates $f = n/d$ for all the algorithms tested averaged over all trials for a fixed image from MNIST (the digit `0') in Figure \ref{fig:cs_img} (c). We note that both Net-GD and Net-PGD produce superior reconstructions as compared to state of art. Running time performance is reported in Appendix \ref{sec:addexp}.

\subsection{Compressive phase retrieval}

\textit{Algorithms and baselines:}
We implement 3 schemes based on \textit{untrained} priors for solving CPR , (i) Net-GD (ii) Net-PGD and finally (iii) Sparse Truncated Amplitude Flow (Sparta) \cite{sparta}, with sparse prior in DCT basis for both datasets. The reconstructions are shown in Figure  \ref{fig:pr_img} for (a) MNIST and (b) CelebA datasets. We plot nMSE at varying compression rates for all algorithms averaged over all trials for a fixed image from MNIST (the digit `0') in Figure \ref{fig:pr_img}(c) and note that both Net-GD and Net-PGD outperform Sparta. 
Running term performance as well as goodness of random initialization scheme are discussed in Appendix \ref{sec:addexp}.


\bibliographystyle{unsrt}
\bibliography{biblio}

\begin{thebibliography}{10}

\bibitem{vincent2010stacked}
P.~Vincent, H.~Larochelle, I.~Lajoie, Y.~Bengio, and P.~Manzagol.
\newblock Stacked denoising autoencoders: Learning useful representations in a
  deep network with a local denoising criterion.
\newblock {\em Journal of machine learning research}, 11(Dec):3371--3408, 2010.

\bibitem{srcnn}
C.~Dong, C.~Loy, K.~He, and X.~Tang.
\newblock Image super-resolution using deep convolutional networks.
\newblock {\em IEEE transactions on pattern analysis and machine intelligence},
  38(2):295--307, 2016.

\bibitem{onenet}
J.~Chang, C.~Li, B.~P{\'o}czos, and B.~Kumar.
\newblock One network to solve them all—solving linear inverse problems using
  deep projection models.
\newblock In {\em 2017 IEEE International Conference on Computer Vision
  (ICCV)}, pages 5889--5898. IEEE, 2017.

\bibitem{prdeep}
C.~Metzler, P.~Schniter, A.~Veeraraghavan, and R.~Baraniuk.
\newblock prdeep: Robust phase retrieval with a flexible deep network.
\newblock In {\em International Conference on Machine Learning}, pages
  3498--3507, 2018.

\bibitem{gan}
I.~Goodfellow, J.~Pouget-Abadie, M.~Mirza, B.~Xu, D.~Warde-Farley, S.~Ozair,
  A.~Courville, and Y.~Bengio.
\newblock Generative adversarial nets.
\newblock In {\em Advances in neural information processing systems}, pages
  2672--2680, 2014.

\bibitem{CSGAN}
A.~Bora, A.~Jalal, E.~Price, and A.~Dimakis.
\newblock Compressed sensing using generative models.
\newblock In {\em Proceedings of the 34th International Conference on Machine
  Learning-Volume 70}, pages 537--546. JMLR. org, 2017.

\bibitem{PRGAN}
P.~Hand, O.~Leong, and V.~Voroninski.
\newblock Phase retrieval under a generative prior.
\newblock In {\em Advances in Neural Information Processing Systems}, pages
  9136--9146, 2018.

\bibitem{deepcs}
T.~Lillicrap Y.~Wu, M.~Rosca.
\newblock Deep compressed sensing.
\newblock {\em arXiv preprint arXiv:1905.06723}, 2019.

\bibitem{DIP}
D.~Ulyanov, A.~Vedaldi, and V.~Lempitsky.
\newblock Deep image prior.
\newblock In {\em Proceedings of the IEEE Conference on Computer Vision and
  Pattern Recognition}, pages 9446--9454, 2018.

\bibitem{DD}
R.~Heckel and P.~Hand.
\newblock Deep decoder: Concise image representations from untrained
  non-convolutional networks.
\newblock In {\em International Conference on Learning Representations}, 2018.

\bibitem{bpdn}
S.~Chen, D.~Donoho, and M.~Saunders.
\newblock Atomic decomposition by basis pursuit.
\newblock {\em SIAM review}, 43(1):129--159, 2001.

\bibitem{bm3d}
K.~Dabov, A.~Foi, V.~Katkovnik, and K.~Egiazarian.
\newblock Image denoising with block-matching and 3d filtering.
\newblock In {\em Image Processing: Algorithms and Systems, Neural Networks,
  and Machine Learning}, volume 6064, page 606414. International Society for
  Optics and Photonics, 2006.

\bibitem{CSC}
V.~Papyan, Y.~Romano, J.~Sulam, and M.~Elad.
\newblock Convolutional dictionary learning via local processing.
\newblock In {\em Proceedings of the IEEE International Conference on Computer
  Vision}, pages 5296--5304, 2017.

\bibitem{timeseriesdnp}
A.~Dimakis S.~Ravula.
\newblock One-dimensional deep image prior for time series inverse problems.
\newblock {\em arXiv preprint arXiv:1904.08594}, 2019.

\bibitem{audiodnp}
L.~Wolf M.~Michelashvili.
\newblock Audio denoising with deep network priors.
\newblock {\em arXiv preprint arXiv: arXiv:1904.07612}, 2019.

\bibitem{MLCSC}
J.~Sulam, V.~Papyan, Y.~Romano, and M.~Elad.
\newblock Multilayer convolutional sparse modeling: Pursuit and dictionary
  learning.
\newblock {\em IEEE Transactions on Signal Processing}, 66(15):4090--4104,
  2018.

\bibitem{csdip}
D.~Van~Veen, A.~Jalal, E.~Price, S.~Vishwanath, and A.~Dimakis.
\newblock Compressed sensing with deep image prior and learned regularization.
\newblock {\em arXiv preprint arXiv:1806.06438}, 2018.

\bibitem{cs}
D.~Donoho.
\newblock Compressed sensing.
\newblock {\em IEEE Transactions on information theory}, 52(4):1289--1306,
  2006.

\bibitem{cosamp}
D.~Needell and J.~Tropp.
\newblock Cosamp: Iterative signal recovery from incomplete and inaccurate
  samples.
\newblock {\em Applied and computational harmonic analysis}, 26(3):301--321,
  2009.

\bibitem{modelcs}
R.~Baraniuk, V.~Cevher, M.~Duarte, and C.~Hegde.
\newblock Model-based compressive sensing.
\newblock {\em IEEE Transactions on Information Theory}, 56:1982--2001, 2010.

\bibitem{copram}
G.~Jagatap and C.~Hegde.
\newblock Fast, sample-efficient algorithms for structured phase retrieval.
\newblock In {\em Advances in Neural Information Processing Systems}, pages
  4917--4927, 2017.

\bibitem{sparta}
G.~Wang, L.~Zhang, G.~Giannakis, M.~Ak{\c{c}}akaya, and J.~Chen.
\newblock Sparse phase retrieval via truncated amplitude flow.
\newblock {\em IEEE Transactions on Signal Processing}, 66(2):479--491, 2017.

\bibitem{ptychogan}
F.~Shamshad, F.~Abbas, and A.~Ahmed.
\newblock Deep ptych: Subsampled fourier ptychography using generative priors.
\newblock In {\em IEEE International Conference on Acoustics, Speech and Signal
  Processing (ICASSP)}, pages 7720--7724. IEEE, 2019.

\bibitem{csdd}
R.~Heckel.
\newblock Regularizing linear inverse problems with convolutional neural
  networks.
\newblock {\em arXiv preprint arXiv:1907.03100}, 2019.

\bibitem{tval3}
C.~Li, W.~Yin, and Y.~Zhang.
\newblock User’s guide for tval3: Tv minimization by augmented lagrangian and
  alternating direction algorithms.

\bibitem{wavelet}
S.~Mallat.
\newblock {\em A wavelet tour of signal processing}.
\newblock Elsevier, 1999.

\bibitem{DL}
M.~Aharon, M.~Elad, and A.~Bruckstein.
\newblock K-svd: An algorithm for designing overcomplete dictionaries for
  sparse representation.
\newblock {\em IEEE Transactions on signal processing}, 54(11):4311, 2006.

\bibitem{hyder2019alternating}
R.~Hyder, V.~Shah, C.~Hegde, and S.~Asif.
\newblock Alternating phase projected gradient descent with generative priors
  for solving compressive phase retrieval.
\newblock {\em arXiv preprint arXiv:1903.02707}, 2019.

\bibitem{shah2018solving}
V.~Shah and C.~Hegde.
\newblock Solving linear inverse problems using gan priors: An algorithm with
  provable guarantees.
\newblock In {\em 2018 IEEE International Conference on Acoustics, Speech and
  Signal Processing (ICASSP)}, pages 4609--4613. IEEE, 2018.

\bibitem{jagatap2019sample}
G.~Jagatap and C.~Hegde.
\newblock Sample-efficient algorithms for recovering structured signals from
  magnitude-only measurements.
\newblock {\em IEEE Transactions on Information Theory}, 2019.

\bibitem{twf}
Y.~Chen and E.~Candes.
\newblock Solving random quadratic systems of equations is nearly as easy as
  solving linear systems.
\newblock In {\em Advances in Neural Information Processing Systems}, pages
  739--747, 2015.

\bibitem{thwf}
T.~Cai, X.~Li, and Z.~Ma.
\newblock Optimal rates of convergence for noisy sparse phase retrieval via
  thresholded wirtinger flow.
\newblock {\em The Annals of Statistics}, 44(5):2221--2251, 2016.

\bibitem{DCGAN}
A.~Radford, L.~Metz, and S.~Chintala.
\newblock Unsupervised representation learning with deep convolutional
  generative adversarial networks.
\newblock {\em arXiv preprint arXiv:1511.06434}, 2015.

\bibitem{oymak2019towards}
S.~Oymak and M.~Soltanolkotabi.
\newblock Towards moderate overparameterization: global convergence guarantees
  for training shallow neural networks.
\newblock {\em arXiv preprint arXiv:1902.04674}, 2019.

\bibitem{du2018gradient}
S.~Du, X.~Zhai, B.~Poczos, and A.~Singh.
\newblock Gradient descent provably optimizes over-parameterized neural
  networks.
\newblock In {\em International Conference on Learning Representations}, 2018.

\bibitem{zhang2016reshaped}
H.~Zhang and Y.~Liang.
\newblock Reshaped wirtinger flow for solving quadratic system of equations.
\newblock In {\em Advances in Neural Information Processing Systems}, pages
  2622--2630, 2016.

\bibitem{rwf}
Huishuai Zhang and Yingbin Liang.
\newblock Reshaped wirtinger flow for solving quadratic system of equations.
\newblock In {\em Advances in Neural Information Processing Systems}, pages
  2622--2630, 2016.

\bibitem{Sarls2006ImprovedAA}
Tam{\'a}s S.
\newblock Improved approximation algorithms for large matrices via random
  projections.
\newblock {\em 2006 47th Annual IEEE Symposium on Foundations of Computer
  Science (FOCS'06)}, pages 143--152, 2006.

\end{thebibliography}

\clearpage
\appendix

\section{Projection to deep network prior} \label{sec:projection}
The projection steps in both Algorithms \ref{algo:pgd} and \ref{algo:pgdpr} represent the problem of fitting an image to an untrained neural network representation. This is the original setting for denoising and compression applications in \cite{DIP} and \cite{DD}. The algorithmic approach to solving this problem is via standard solvers such as gradient descent (GD) or  Adam. The problem takes the form:
\begin{align} \label{eq:min3}
\min_{\w}\mathcal{L}(\w;z,v) := \min_{\w}\|v - G(\w;z)\|^2 ,
\end{align}
where $v$ is typically a noisy variant of the original image $x^*$. The problem in Eq.\ref{eq:min3} is non-convex due to the structure of $G(\w;z)$. Convergence guarantees for deep neural network formulations of this form that exist are highly restrictive \cite{oymak2019towards,du2018gradient}. There exist several papers in recent literature which allude to (linear) convergence of gradient descent for solving the two-layer neural networks; however all of the results rely on moderate or extreme overparameterization of the neural network. Therefore, these results do not apply to our paper and deriving convergence guarantees for the denoising problem in \ref{eq:min3} is an interesting direction for future work. 
\section{Discussion on sample complexity} \label{sec:discussion}
In compressive imaging literature, for $s$-sparse signals of dimension $d$, the sample complexity for compressive sensing is $n=O(s\log d)$ and compressive phase retrieval is $n=O(s^2 \log d)$, when Gaussian measurements are considered. If structural constraints are imposed on the sparsity of images, such as block sparsity, the sample requirements can be brought down to $n=O(s/b \log d)$ and $n=O(s^2/b \log d)$ for CS and CPR respectively, where $b$ is the block length of each sparse block \cite{copram}. However these gains come at the cost of designing the signal priors carefully.

In contrast, the sample requirements with deep network priors, as we show in this paper is $n=O(k_1k_2 \log d)$ for a two-layer setup. In both datasets that we tested, relatively shallow architectures were sufficient. Therefore the effective sample complexity is of the order of $k_1 k_2$, which is typically much smaller than the dimension $d$. We have empirically demonstrated in Section \ref{sec:exps} that the sample requirement with deep network priors is significantly lower than that for the sparse prior setting. Moreover, the design of the prior is fairly straightforward, and applies for a wide class of images. 
\section{Additional experiments} \label{sec:addexp}
In this section we present some additional details for the experimental setup in Section \ref{sec:exps}. We also present some additional experiments to reinforce the merits of Net-PGD. 

All codes were run on a Nvidia GeForce GPU with 8GB RAM. 

\textit{Deep network architecture:}
For both MNIST and CelebA images, several architectures were tried out to pick out the best under-parameterized network which gave low representation error. We found that for the example images from MNIST, a decoder architecture, as described in Eq. \ref{eq:decoder} with 2 layers, and channel configurations $k_1=15,k_2=15,k_3 =10$ and bilinear upsampling operators each with upsampling factor of $2$, $U_l^{\uparrow 2}, l=\{1,2,3\}$ was sufficient to represent most images. The outputs after each ReLU operation are normalized, by calling for batch normalization subroutine in Pytorch. Finally a sigmoid activation is added to the output of the deep network, which smoothens the output; however this is not mandatory for the deep network configuration to work. For CelebA images, we fixed the configuration to a 3 layer network with setup $k_1=120,k_2=15,k_3=15,k_4=10$. Note that both of these architectures are \textit{underparameterized}, unlike the configurations in \cite{DIP}. The random seed $Z_1$ is fixed and picked from uniform random distribution \footnote{Gaussian distributed entries as well as randomly picked rows of Hadamard matrices also work.}. We plot the ``compressed" representations of each image, $G(\w;z)$ in all Figures for reference.

\subsection{Compressed sensing recovery}

\textit{Implementation details:} For CS recovery with deep network priors, both Net-GD and Net-PGD were implemented using the PyTorch framework with Python 3 and using GPU support. For Net-GD, SGD (alternatively, Adam) optimizer is used. For Net-PGD, SGD (alternatively, Adam) optimizer is used for the projection step and SGD optimizer for the gradient step in Step 3 of Alg. \ref{algo:pgd} and Step 4 of Alg. \ref{algo:pgdpr}. For implementing Lasso algorithm, Python's \texttt{sklearn.linear\_model} library was used and we set the regularization factor $\alpha=10^{-5}$. The MATLAB code for TVAL3 \cite{tval3} made available on the author's website was used  with its default settings.  

\textit{Running time:}  We also report the average running times for different algorithms across different measurement levels for examples from MNIST is  5.86s (Net-GD),  5.46s (Net-PGD),  2.43s (Lasso-DCT), 0.82s (TVAL3). We note that the running time of both GD and PGD for CS-UNP are competitive. 
\subsection{Compressive phase retrieval}
\textit{Implementation details:} For compressive phase retrieval with deep network priors, both Net-GD and Net-PGD were implemented using the PyTorch framework with Python 3 and using GPU support. All optimization procedures were implemented using SGD optimizer. For implementing Sparta algorithm, the algorithm from \cite{sparta} was implemented in MATLAB.   

We also report the average running times for different algorithms across different measurement levels for examples from MNIST is  25.59s (Net-GD),  28.46s (Net-PGD),  3.80s (Sparta-DCT). 

\textit{Goodness of random initialization:} Our theoretical guarantees for phase retrieval hold only as long as the initialization $x^0$ is close to the ground truth $x^*$. We perform rigorous experiments to assert that uniform random initialization of the weights $\w^0$ of the neural network, ensure that the initial estimate $\x^0 = G(\w^0;z)$ is good. We denote the distance of initialization as $\delta_i = \|x^0-x^T\|/\|x^T\|$ ($x^T=\hat{x}$) and report the values of $\delta_i$ for the trials in which $\|x^T-x^*\|/\|x^*\| < 0.1$. We plot the average values of $\delta_i$ in Table \ref{tab:init2}.

\begin{table}[!h]
	\centering 
	\caption{Distance of initial estimate $x^0$} \label{tab:init2}
	\begin{tabular}{|c|c|c|c|c|}
		\hline
		n/d & d & channel configuration & nMSE of $\hat{x}$ & average $\delta_i$ values\\
		\hline
		0.2 & 784 (MNIST) & 15, 15, 10 & 0.098 & 0.914\\
		0.5 & 784 (MNIST) & 15, 15, 10 & 0.018 & 0.942\\
		\hline
		0.4 & 12288 (CelebA) & 120, 15, 15,10 & 0.020 &0.913\\
		0.6 & 12288 (CelebA) & 120, 15, 15,10 & 0.015 &0.915\\
		\hline
	\end{tabular}
\end{table}
From our observation, uniform random initialization suffices to ensure that the conditions for Theorem \ref{thm:pr-converge} are met and $\delta_i <1$.

\newpage

\section{Proofs and supporting lemmas} \label{sec:appendix}

In this section we proofs for the theorems discussed in the main body of this paper as well as present supporting Lemmas. 

We first discuss the set-restricted restricted isometry property.

The $(\S,\gamma,\beta)$ REC holds for Gaussian matrix $A$ with high probability, as long as certain dimensionality requirements are met. We show this via Lemma \ref{lem:leftRIP} as follows:

\lemRIP*

\begin{proof}
	
	We first describe the two layer setup.	
	
	Consider the action of measurement matrix $A$ defined on vector $h$, where $h:=U_1\sigma (ZW_1)W_2$ below:
	$$u = Ah = AU_1\sigma(Z_1W_1)W_2.$$
	where $W_1^{k_1\times k_2}$, $W_1^{k_2\times 1}$ and $U_1^{d\times d_1}$ with $d>d_1$.
	 
	We would like to estimate the dimensionality of $A$, required to ensure that the action of $A$ on set restricted vector $h\in \S$, is bounded as:
	$$ 
	\gamma\|h\|^2 \leq \|Ah\|^2 \leq \beta \|h\|^2
	$$
	with high probability. To establish this, consider the following argument which is similar to the union of subspaces argument from \cite{CSGAN}.
	 
	The action of ReLU on input $(Z_1W_1)$ partitions the input space of variable $W_1$ into a union of linear subspaces. In particular, consider a single column of $w_{1,:,j}$ of ${W_1}$, indexed by $j$, which is $k_1$ dimensional. Then, $\sigma(Z_1w_{1,j})$ partitions the $k_1$-dimensional input space into $(d_1^{k_1})$ $k_1$-spaces. Since there are $k_2$ such columns, effectively the $k_1\times k_2$ dimensional space of $W_1$ is partitioned into $(d_1^{k_1})^{k_2}$, $(k_1\times k_2)$-spaces. 
	
	Then, we can consider the union of $d_1^{k_1 k_2}$ subspaces with linearized mappings of the form (weights are indexed as $[layer, row, column]$:
	\begin{align} \nonumber
u = AU_1(P_1\circ (Z_1W_1))W_2 &= AU_1[P_{1,1}\cdot (Z_1w_{1,:,1})\dots P_{1,k_2}\cdot (Z_1w_{1,:,k_2})]W_2 \\ \nonumber
&= AU_1\sum_{j=1}^{k_2} P_{1,j}Z_1(w_{2,j,1}\cdot w_{1,:,j}) \\ \nonumber
&= AU_1 [P_{1,1}Z_1 \cdots P_{1,k_2}Z_1] \cdot [w_{2,1,1}\cdot w_{1,:,1}^\top \dots w_{2,k_2,1}\cdot w_{1,:,k_2}^\top]^\top \\ \label{eq:subsp}
&= AU_1 Z_1' \cdot w
	\end{align}
	
	where $P_1$ is the fixed sign pattern of the ReLUs, $P_{1,j}$, $j\in \{1\dots k_2\}$ is a diagonal matrix with the 0 or 1 corresponding to $k_2$ different neurons, each of $(w_{2,j,1}\cdot w_{1,:,j})$ is a $k_1$-dimensional vector, and there are $k_2$ such vectors. Each $P_{1,j}$ partitions the $k_1$-dimensional space it acts on, into $d_1^{k_1}$ subspaces. Since there are $k_2$ such vectors, the total number of partitions is $(d_1^{k_1})^{k_2}$ total partitions. For fixed $P_1$, $Z_1' := [P_{1,1}Z_1 \cdots P_{1,k_2}Z_1]$.
	
	
	If the dimensionalities are chosen such that they satisfy $d>k_1k_2$, and $A, U_1,Z_1$ are known matrix operators, then, for a fixed sign pattern $P_1$, the $w^{k_1k_2 \times 1} := [w_{2,1,1}\cdot w_{1,:,1}^\top \dots w_{2,k_2,1}\cdot w_{1,:,k_2}^\top]^\top $ represents the accumulated weights, belonging to one of the $d_1^{k_1 k_2}$ subspaces and 	 $(U_1Z_1')^{d\times k_1k_2}$ is a linear transformation from a lower dimensional space to a higher dimensional space. Then, if $A$ is designed as an oblivious subspace embedding (OSE) (Lemma \ref{lem:ose} in Appendix \ref{sec:appendix}) of $U_1Z_1'w$, for a single $k_1k_2$-dimensional subspace of $w$, one requires $n = O\left( \frac{k_1k_2}{\alpha^2} \right) $ samples to embed the vector $h = U_1Z_1'w$, as 
	\begin{align} \label{eq:bound}
		(1-\alpha)\|h\|^2 \leq \|Ah\|^2 \leq (1+\alpha) \|h\|^2,
	\end{align}

	with probability $1-e^{-c\alpha_1^2 n}$, for constant $\alpha_1 <\alpha$. Since there are $d_1^{k_1 k_2}$ such subspaces, then for the OSE to hold for all subspaces, one requires to take a union bound as $1-d_1^{k_1 k_2} e^{-c\alpha_1^2 n}.$ 
	Therefore the expression in Eq. \ref{eq:bound} holds for all $h \in \S$, with probability $1-e^{-c\alpha_2^2n}$ and $\alpha_2 < \alpha_1$.
	Therefore, one requires $n = O\left(\frac{k_1 k_2 \log d_1}{\alpha_1^2}\right)$, to ensure that $A$ satisfies $(\S,1-\alpha,1+\alpha)$-REC with probability $1-e^{-c\alpha_2^2 n}$.

\end{proof}                                                                     
 Next, we present some corollaries which will be useful for proving some of our theoretical claims.                                  
\begin{corollary} \label{cor:diff}
	For parameter $\alpha >0$, if a matrix $A \in \mathbb{R}^{n\times d}$ satisfies $(\S,1-\alpha,1+\alpha)$-REC with probability $1-e^{-c\alpha_o^2n}$, for all $x \in \S$, then for $x_1,x_2 \in \S$, 
	$$ (1-\alpha)\|x_1-x_2\|^2 \leq \|A(x_1-x_2)\|^2 \leq (1+\alpha)\|x_1-x_2\|^2 ,$$
	holds with probability $1-e^{-c_2\alpha_o^2 n}$, where $c_2<c$.
\end{corollary}

\begin{proof}
Since $x_1,x_2 \in \S$, both $x_1,x_2$ lie in the union of $k_1$-dimensional subspaces, the difference vector $x_3 = x_1 - x_2 \in \S'$, lies in a union of $2k_{1}$-dimensional subspaces.  For $(\S,1-\alpha,1+\alpha)$-REC to hold for the difference set, one continues to require $n = O\left(\frac{k_1k_2 \log d}{\alpha_1^2}\right)$. 
\end{proof}

\begin{corollary} \label{cor:eig}
	If $A$ satisfies set-restricted REC and $h^t = x^t - x^*$, with $x^t,x^* \in \S$ then
\begin{align*}
\|(1-\eta A^\top A)h^t\| &\leq \max \{1-\eta \lambda_{min}, \eta \lambda_{max} -1\} \|h^t\| 
\end{align*}
with $\lambda_{min} = (1-\alpha)$ and $\lambda_{max} = (1+\alpha)$.
\end{corollary}

\begin{proof}
Consider $h \in \S'$, where $h = h^t = x^t - x_2$ and $x^t,x^* \in \S$. Then from Set-REC and Corollary \ref{cor:diff},
$$
(1-\alpha)\|h\|^2 \leq \|Ah\|^2 \leq (1+\alpha)\|h\|^2.
$$
From Eq. \ref{eq:subsp}, if $x_1,x_2 \in \S$, then it is possible to write $h$ to arise from a union of $2k_1$-dimensional subspaces of the form $
h = {B} w$ ($B = U'Z'$).
Then, 
\begin{align} \label{eq5}
(1-\alpha)\|{B}w\|^2 \leq \|A{B}w\|^2 \leq (1+\alpha)\|{B}w\|^2.
\end{align}
where $w\in \reals^{2k_1}$. We need to evaluate the eigenvalues of $\|A^\top A \|$ restricted on set $\S'$, which we can do by inducing a projection on the union of subspaces ${B}$ as
$$
\|A^\top A h\| = \|{B}^\top A^\top A {B} w\|
$$
Therefore, the minimum and maximum eigenvalues of $\|A^\top A\|$ restricted on set $\S'$ are
$$\sigma_{min}(A{B}) \leq\|{B}^\top A^\top A\bar{B}\|_2 \leq \sigma_{max}(A{B})$$

Then, using Eq.\ref{eq5}, $
(1-\alpha)\sigma_{min}({B}) \leq \|{B}^\top A^\top A{B}\|_2 \leq (1+\alpha)\sigma_{max}({B})
$.

Since $B$ predominantly consists of a product of upsampling matrices and latent code $Z_1$, which can be always chosen such that $\sigma_{max}(Z_1)\approx\sigma_{min}(Z_1)$, therefore $\sigma_{max}(B) \approx \sigma_{min}(B) \approx 1$.
\end{proof}

Next, we discuss the convergence of Net-PGD for compressed sensing recovery via Theorem \ref{thm:cs-converge}. 

\thmCS*

\begin{proof}

	Using the definition of loss as $L(x^{t}) = \|y-Ax^t\|^2$, 
	\begin{align}\nonumber
	L(x^{t+1}) - L(x^t) &= (\|Ax^{t+1}\|^2 - \|Ax^t\|^2) -  2(y^\top Ax^{t+1} - y^\top A x^{t})   \\ \nonumber
	&= \|Ax^{t+1} - Ax^t\|^2 - 2(Ax^t)^\top (Ax^t) + 2(Ax^{t})^\top (Ax^{t+1})  \\ \nonumber
	& \hspace{5cm}- 2(y^\top Ax^{t+1} - y^\top A x^{t})\\
	&= \|Ax^{t+1} - Ax^t\|^2 - 2(y-Ax^t)^\top(Ax^{t+1}-Ax^t) \label{eq:1}
	\end{align}
	We want to establish a contraction of the form $L(x^{t+1}) < \nu L(x^t)$, with $\nu<1$.
	
	Step 3 of Alg. \ref{algo:pgd} is solved via gradient descent:
	\begin{align} \label{eq:gd}
	v^t = x^t - \eta A^\top (Ax^t-Ax^*) 
	\end{align}
	Subsequently, Step 4 of Algorithm \ref{algo:pgd} learns weights $\w^t$ that produce $x^t = G(\w^t;z)$, which lies in the range of the decoder $G(\cdot)$ and is closest to the estimate $v^t$.

	Step 4 of Algorithm \ref{algo:pgd} produces an update of $\w^t$ satisfying:
	\begin{align*}
	\|G(\w^t;z)-v^t\| \leq \|G(\w^*;z)-v^t\| 
	\end{align*}

	Denoting $G(\w^t;z):=x^t$ and $G(\w^*;z):=x^*$, and using the update rule in Eq. \ref{eq:gd}, 
	
	\begin{align*}
	\|x^{t+1}-v^t\|^2 &\leq \|x^*-v^t\|^2\\
	\|x^{t+1}-x^t+\eta A^\top A(x^t-x^*)\|^2 &\leq \|x^*-x^t + \eta A^\top A(x^t-x^*)\|^2\\
	\|x^{t+1}-x^t\|^2 + 2\eta (A(x^t-x^*))^\top A(x^{t+1}-x^*) &\leq \|x^{t}-x^*\|^2 - 2\eta \|A(x^t-x^*)\|^2\\
	\frac{1}{\eta}\|x^{t+1}-x^t\|^2 + 2(A(x^t-x^*))^\top A(x^{t+1}-x^*) &\leq \frac{1}{\eta}\|x^{t}-x^*\|^2 - 2L(x^t)\\
	\implies	L(x^{t+1}) + L(x^t)&\leq \frac{1}{\eta} \|x^t - x^{*}\|^2   - \frac{1}{\eta}\|x^{t+1} - x^{t}\|^2 \\&\hspace{3cm}+ \|A(x^{t+1}-x^t)\|^2 
	\end{align*}
	where we have used the expansion in Eq. \ref{eq:1}.  We now use $(\S,\gamma,\beta)$-REC. If a Gaussian measurement matrix is considered then $\gamma = 1-\alpha$ and $\beta = 1+\alpha$.
	
	Using $(S,\gamma)$-REC on the first term on the right side, 
	\begin{align*}
	\|x^{*} - x^t\|^2 \leq \frac{1}{\gamma}\|A(x^{*} - x^t)\|^2
	\end{align*} 
	Second, using $(S,\beta)$-REC on the last term on the right side,
	\begin{align*}
	\|A(x^{t+1}-x^t)\|^2 \leq \beta \|x^{t+1}-x^t\|^2 
	\end{align*}
	Accumulating these expressions and substituting, 
	\begin{align*}
	L(x^{t+1}) + L(x^t)&\leq \frac{1}{\eta\gamma}	L(x^t) + \left(\beta-\frac{1}{\eta}\right) \|x^{t+1}-x^t\|^2\\
	&\stackrel{\beta\eta<1}{\leq} \frac{1}{\eta\gamma^2}	L(x^t) 
	\\
	\implies L(x^{t+1}) &\leq \nu L(x^t) \\
	\implies L(x^{T}) &\leq \nu^T L(x^0) 
	\end{align*}
	where $0 < \nu < 1$ and $\nu = \left(\frac{1}{\eta\gamma^2}-1\right)$ and picking $\eta <1/\beta$.  Invoking $(S,\gamma,\beta)$-REC again,
	\begin{align*}
	\|x^T - x^*\|^2 \leq \frac{1}{\gamma} \|y - Ax^T\|^2 \leq \frac{\nu^T}{\gamma} \|y - Ax^0\|^2 := \epsilon
	\end{align*}
	Hence to reach $\epsilon$- accuracy in reconstruction, one requires $T$ iterations where 
	$$T = \log_\alpha\left(\frac{\|y - Ax^0\|^2}{\gamma \epsilon}\right).$$
	Note that the contraction $L(x^{t+1}) \leq \nu  L(x^{t})$ coupled with $(\S,\gamma,\beta)$-REC implies a distance contraction $\|x^{t+1}-x^*\| \leq \nu_o  \|x^{t}-x^*\|$, with $\nu_o = \nu\sqrt{\beta/\gamma}$. 
	
	Step 4 of Algorithm \ref{algo:pgd}, which is essentially the case of fitting a noisy image to a deep neural network prior can be solved via gradient descent. %
	We discuss this projection in further detail in Section \ref{sec:projection}.
\end{proof}

Next, we discuss the main convergence result of Net-PGD for compressive phase retrieval in Theorem \ref{thm:pr-converge}.

\thmPR*

\begin{proof}
	Step 4 of Algorithm \ref{algo:pgdpr} is solved via a variant of gradient descent called Wirtinger flow \cite{zhang2016reshaped}, which produces updates of the form:
	\begin{align} \nonumber
	v^{t+1} &= x^t - \eta A^\top \left(Ax^t - Ax^* \circ \text{sign}(Ax^*) \circ\text{sign}(Ax^t)\right) \\ \nonumber
	&=x^t - \eta A^\top \left(Ax^t - Ax^*\right) - \eta A^\top A x^* \circ (1- \text{sign}(Ax^*) \circ\text{sign}(Ax^t))   \\ \label{eq:prstep}
	&= x^t - \eta A^\top \left(Ax^t - Ax^*\right) - \eta \varepsilon_p^t
	\end{align}
	where $\varepsilon_p^t :=  A^\top  A x^*\circ(1- \text{sign}(Ax^*) \circ\text{sign}(Ax^t))$ is \textit{phase estimation} error.
	
	If $\text{sign}(Ax^*) \approx \text{sign}(Ax^t)$, then the above resembles the gradient step from the linear compressed sensing formulation. Thus, if $x^0$ is initialized well, the error due to phase mis-match $\varepsilon_p^t$ can be bounded, and subsequently, a convergence result can be formulated.
	
	Next, Step 4 of Algorithm \ref{algo:pgdpr} learns weights $\w^t$ that produce $x^t = G(\w^t;z)$, which lies in the range of the decoder $G(\cdot)$ and is closest to the estimate $v^t$. We discuss this projection in further detail in Appendix \ref{sec:projection}.
	
	Step 4 of Algorithm \ref{algo:pgdpr} produces an update of $\w^t$ satisfying:
	\begin{align*}
	\|G(\w^t;z)-v^t\| &\leq \|G(\w^*;z)-v^t\| \\
	\equiv \|x^t-v^t\| &\leq \|x^*-v^t\| 
	\end{align*}
	for $t=\{1,2,\dots T\}$. Then, the above projection rule yields:
	$$
	\|x^{t+1} - v^{t+1} + v^{t+1} -  x^* \| \leq \|x^{t+1} - v^{t+1}\| + \|x^* - v^{t+1}\| \leq 2 \|x^*-v^{t+1}\|
	$$
	Using the update rule from Eq. \ref{eq:prstep} and plugging in for $v^{t+1}$:
	\begin{align*}
	\frac{1}{2}\|x^{t+1}-x^*\|^2 &\leq \|(x^t-x^*)  - (\eta A^\top \left(Ax^t - Ax^*\right) + \eta \varepsilon_p^t)\|^2 
	\end{align*}
	Defining $h^{t+1} = x^{t+1}-x^*$ and $h^{t} = x^{t}-x^*$, the above expression is
	\begin{align}
	\label{eq:converge}
	\frac{1}{2}\|h^{t+1}\| \leq \|h^t  - \eta A^\top Ah^t - \eta \varepsilon_p^t\| \leq \|(1-\eta A^\top A) h^t\| + \eta\|\varepsilon_p^t\|
	\end{align}
	
	We now bound the two terms in the expression above separately as follows. The first term is bounded using matrix norm inequalities Using Corollary \ref{cor:eig} (in Appendix \ref{sec:appendix}) of $(\S,\gamma,\beta)$-REC: 
	\begin{align*}
	\|(1-\eta A^\top A)h^t\| &\leq \max \{1-\eta \lambda_{min}, \eta \lambda_{max} -1\} \|h^t\| 
	\end{align*}
	where $\lambda_{min}$ and $\lambda_{max}$ are the minimum and maximum eigenvalues of $A^\top A$ restricted on set $\S$, and via Corollary \ref{cor:eig}, $\lambda_{min} = (1-\alpha)$, $\lambda_{max} = (1+\alpha)$.
	
	Hence the first term in the right side of Eq.\ref{eq:converge} is bounded as:
	\begin{align*}
	\label{term1}
	\|(1-\eta A^\top A)h^t\| &\leq  \rho_o \|h^t\|.
	\end{align*}
	where $\rho_o = \max \{1-\eta (1-\alpha) , \eta (1+\alpha) -1\} $. The second term in Eq.\ref{eq:converge} is bounded via Lemma \ref{lem:phaseerr} as follows:
	$$\|\varepsilon_p^t\| \leq \delta_o \|x^t-x^*\|$$
	as long as $\|x^0-x^*\| \leq \delta_i \|x^*\|$. 
	
	Substituting back in Eq.\ref{eq:converge},
	\begin{align*}
	\|x^{t+1}-x^*\| \leq 2(\rho_o + \eta\delta_o)\|x^t - x^*\| := \rho \|x^t - x^*\|.
	\end{align*}
	Then, if we pick constant $\eta = \frac{1}{1+\alpha+1-\alpha} =1$ that minimizes $\rho := 2(\max\{1-\eta (1-\alpha) , \eta (1+\alpha) -1\} + \eta\delta_o)$, to yield $\rho = 2(\alpha + \delta_o)$  then we obtain the linear convergence criterion as follows:
	$$
	\|x^{t+1}-x\| \leq \rho \|x^t - x\|.
	$$
	Here, if we set $\alpha=0.1$ and $\delta_o=0.36$ from Lemma \ref{lem:phaseerr}, then $\rho = 0.92 < 1$.
	Note that this proof relies on a bound on the phase error $\|\varepsilon_p^t\|$ which is established via Lemma \ref{lem:phaseerr} as follows:
	
	\begin{lemma}\label{lem:phaseerr}
		Given initialization condition $\|x^0-x^*\| \leq \delta_i \|x^*\|$, then if one has Gaussian measurements $A \in \reals^{n \times d}$ such that $n =O\left({k_1} k_2 \log d\right)$ for a two-layer decoder, then with probability $1-e^{- c_2 n}$ , the following holds:
		$$
		\|\varepsilon_p^t\| = \|A^\top A x^* \circ (1-\text{sign}(Ax^*)\circ \text{sign}(Ax^t))\| \leq \delta_o \|x^t-x^*\|
		$$
		for constant $ c_2$ and $\delta_o = 0.36$.
	\end{lemma}
	\begin{proof}
		We adapt the proof of Lemma C.1. of \cite{jagatap2019sample} as follows. 
		
		We define indicator function $\mathbf{1}_{(a_i^\top x^t)(a_i^\top x^*)<1} = \frac{1}{2}(1-\text{sign}(Ax^*)\circ \text{sign}(Ax^t))$ with zeros where the condition is false and ones where the condition is true.
		
		Then we are required to bound the following expression:
		\begin{align*}
		\|\varepsilon_p^t\|^2 = 2\sum_{i=1}^{n} (a_i^\top x^*)^2 \cdot \mathbf{1}_{(a_i^\top x^t)(a_i^\top x^*)<1} \leq \delta_o^2 \|x^t-x^*\|^2
		\end{align*}
		
		Following the sequence of arguments in Lemma C.1. of \cite{jagatap2019sample} (or Lemma C.1 of \cite{rwf}), one can show that for a \textit{given} $x^t$, 
		\begin{equation} 
		\|\varepsilon_p^t\|^2 \leq \delta_o^2 + \kappa + \frac{3c_1\kappa}{\delta_i}< 0.13 + \kappa + \frac{3c_1\kappa}{\delta}\label{eq3}
		\end{equation}	
		with high probability, $1-e^{-cn\kappa^2}$, for small constants $c,c_1,\delta$, as long as $\|x^t-x^*\|\leq 0.1\|x^*\|_2$. Here the bound on $\delta_o^2$ (in this case 0.13) is a monotonically increasing function of the distance $\delta_i^t = \frac{\|x^t-x^*\|_2}{\|x^*\|_2}$.
		
		If the projected gradient scheme produces iterates satisfying
		\begin{align*}
		\|x^{t+1} - x^*\| < \rho \|x^t-x^*\|
		\end{align*}
		with $\rho<1$, then the condition in Eq. \ref{eq3} is satisfied for all $t=\{1,2,\dots T\}$ as long as the initialization $x^0$ satisfies $\|x^0-x^*\|\leq 0.1\|x^*\|_2$ (i.e. $\delta_i^0 :=\delta_i = 0.1$). 
		
		Now, the expression in Eq. \ref{eq3} holds for a fixed $x^t$. To ensure that it holds for all possible $x\in \S$, we need to use an epsilon-net argument over the space of variables spanned by $\S$. The cardinality of $\S$ is 
		$$
		\text{card}(\S) < d^{k_1k_2}
		$$
		as seen from the derivation of REC in Lemma \ref{lem:leftRIP}. Therefore,
		$$
		\|\varepsilon_p^t\| \leq 0.13 + \kappa + \frac{3c_1\kappa}{\delta_i}
		$$
		with probability $1-d^{k_1k_2}e^{-cn\kappa^2}$ for small constant $c$. To ensure that high probability result holds for \textit{all} $x\in \S$, 
		\begin{align*}
		e^{k_1k_2 \log d} e^{-cn\kappa^2} &< e^{-c_2 n} \\
		{k_1 k_2 \log d - cn\kappa^2} &< -c_2 n\\
		n > \frac{1}{c\kappa^2 - c_2} k_1 k_2 \log d  &> {c_3k_1} k_2 \log d
		\end{align*}
		for appropriately chosen constants $c,c_2,c_3$.	
	\end{proof}
	Note that this Theorem requires that the weights are initialized appropriately, satisfying $\|x^0-x^*\|\leq \delta_i\|x^*\|$. In Section \ref{sec:exps} we perform rigorous experiments to show that random initialization suffices to ensure that $\delta_i$ is small. 
\end{proof}

Finally we state the statement for Oblivious Subspace Embedding, which is the core theoretical lemma required for proving our REC result.

\begin{lemma} \label{lem:ose}
	Oblivious subspace embedding (OSE) \cite{Sarls2006ImprovedAA}. A $(k,\alpha,\delta)$-OSE is a random matrix $\Pi^{n\times d}$ such that for any fixed $k$-dimensional subspace $\S$ and $x^{d\times 1} \in \S$, with probability $1-\delta$, $\Pi$ is a subspace embedding for $S$ with distortion $\alpha$, where $n = O(\alpha^{-2}(k+\log(\frac{1}{\delta})))$. 
	
	The failure probability is $\delta = e^{-cn\alpha^2 + ck}  $, for small constant $c$ and the embedding satisfies:
	$$(1-\alpha)\|x\|^2 \leq \|\Pi x\|^2 \leq (1+\alpha)\|x\|^2.$$
\end{lemma}

\end{document}